\def\year{2021}\relax
\documentclass[letterpaper]{article} 
\usepackage{aaai21}  
\usepackage{times}  
\usepackage{helvet} 
\usepackage{courier}  
\usepackage[hyphens]{url}  
\usepackage{graphicx} 
\usepackage{natbib}  
\usepackage{caption} 
\usepackage{amsmath}
\usepackage{amsthm}
\usepackage{algorithm,algorithmic}

\newcommand{\BEAS}{\begin{eqnarray*}}
\newcommand{\EEAS}{\end{eqnarray*}}
\newcommand{\BEQ}{\begin{equation}}
\newcommand{\EEQ}{\end{equation}}
\newcommand{\BIT}{\begin{itemize}}
\newcommand{\EIT}{\end{itemize}}

\newcommand{\eg}{{\it e.g.}}
\newcommand{\ie}{{\it i.e.}}
\newcommand{\cf}{{\it cf. }}

\newcommand{\reals}{{\mbox{\bf R}}}

\newcommand{\Expect}{\mathbf{E}}
\newcommand{\prob}{\mathbf{Prob}}
\newcommand{\var}{\mathop{\bf var}}
\newcommand{\V}{\mathop{\bf Var}}

\newcommand{\Var}{\mathop{\bf var{}}}

\newcommand{\dom}{\mathop{\bf dom}}
\newcommand{\argmax}{\mathop{\rm argmax}}

\newtheorem{theorem}{Theorem}
\newtheorem{remark}{Remark}
\newtheorem{assumption}{Assumption}
\newtheorem{lemma}[theorem]{Lemma}
\newtheorem{corollary}[theorem]{Corollary}
\newtheorem{proposition}[theorem]{Proposition}
\theoremstyle{definition}

{\begin{quote}}{\end{quote}}

\makeatletter
\long\def\@makecaption#1#2{
   \vskip 9pt
   \begin{small}
   \setbox\@tempboxa\hbox{{\bf #1:} #2}
   \ifdim \wd\@tempboxa > 5.5in
        \begin{center}
        \begin{minipage}[t]{5.5in}
        \addtolength{\baselineskip}{-0.95pt}
        {\bf #1:} #2 \par
        \addtolength{\baselineskip}{0.95pt}
        \end{minipage}
        \end{center}
   \else
	\hbox to\hsize{\hfil\box\@tempboxa\hfil}
   \fi
   \end{small}\par
}
\makeatother

\newcounter{oursection}

\newcounter{lecture}

\newcommand{\norm}[1]{{\left\vert\kern-0.25ex\left\vert\kern-0.25ex\left\vert #1
    \right\vert\kern-0.25ex\right\vert\kern-0.25ex\right\vert}}

\title{Sample Efficient Reinforcement Learning with  REINFORCE}

\author{
    Junzi Zhang\textsuperscript{\rm 1}, Jongho Kim\textsuperscript{\rm 2}, Brendan O'Donoghue\textsuperscript{\rm 3}, Stephen Boyd\textsuperscript{\rm 2}
}

\affiliations{

    \textsuperscript{\rm 1} Institute for Computational \& Mathematical Engineering, Stanford, USA.\\
    \textsuperscript{\rm 2} Department of Electrical Engineering, Stanford, USA.\\
    \textsuperscript{\rm 3} DeepMind, Google.\\
    junziz@stanford.edu, jkim22@stanford.edu, bodonoghue85@gmail.com, boyd@stanford.edu
}

\begin{document}

\date{\today}
\maketitle

\begin{abstract}
Policy gradient methods are among
the most effective methods for large-scale reinforcement learning,
and their empirical success has
prompted several works that develop the
foundation of their global convergence theory. However, prior
works have either required exact gradients or
state-action visitation measure based mini-batch stochastic gradients with a
\textit{diverging} batch size, which limit their applicability in practical scenarios.
In this paper, we consider
classical policy gradient methods that compute an approximate gradient with
a \textit{single} trajectory or a \textit{fixed size} mini-batch of
trajectories under soft-max parametrization and log-barrier regularization,
along with the
widely-used REINFORCE gradient estimation procedure.
By controlling the number of ``bad'' episodes and 
resorting to the classical doubling trick, 
we establish an anytime sub-linear
high probability regret bound as well as 
almost sure global convergence of the average regret 
with an asymptotically sub-linear rate. 
These provide the first set of global convergence and sample efficiency results
for the well-known REINFORCE algorithm and contribute to a better understanding
of its performance in practice.
\end{abstract}

\section{Introduction}
In this paper, we study the global convergence rates of the
REINFORCE algorithm \cite{REINFORCE} for episodic reinforcement learning.
REINFORCE is a vanilla  
 policy gradient method that computes a stochastic
approximate gradient with a single trajectory or a fixed size
mini-batch of trajectories with particular choice of gradient estimator, where we 
use `vanilla' here to disambiguate the
method from more exotic variants such as natural policy gradient methods. 
REINFORCE and its variants are among the most widely used policy
gradient methods in practice due to their good empirical performance and
implementation simplicity
\cite{mnih2014neural, gu2015muprop, zoph2016neural, rennie2017self, 
guu2017language, johnson2017inferring, 
yi2018neural, 
kool2018attention, 
kool2020estimating}. 
Related methods include the actor-critic family \cite{actor-critic, A3C} and
deterministic and trust-region based variants \cite{DPG, PPO, TRPO}.

The theoretical results for policy gradient methods have, up to recently, been
restricted to convergence to local stationary points \cite{kakade_2019}.
Lately, a series of
works have established \emph{global} convergence results. 
These recent developments
cover a broad range of issues including global optimality characterization
\cite{PG_LQR, russo_PG}, convergence rates \cite{zhang2019global,
softmax_PG, russo_note_2020, cen2020fast}, the use of
function approximation \cite{kakade_2019, neural_PG, fu2020single},
and efficient exploration \cite{PC-PG} (for more details, 
see the related work section, which we defer to Appendix \ref{related_work} 
due to space limits). 
Nevertheless,
prior work on vanilla policy gradient methods  
either requires exact and deterministic 
policy gradients or 
only guarantees convergence up to $\Theta(1/M^p)$ with a fixed
mini-batch
size $M>0$ of trajectories collected when performing a single update 
(where $p>0$ is $1/2$ in most cases),
while global convergence is only achieved when the batch size $M$
goes to infinity. By contrast, practical implementations of
policy gradient methods typically use
either a single or a fixed number of sample trajectories, which
tends to perform well.
In addition, prior theoretical results 
(for general MDPs) 
have used the 
state-action visitation measure based gradient estimation (see \eg,
\cite[(3.10)]{neural_PG}), which 
are typically not used in practice.

The main purpose of this paper is to bridge this gap between theory and
practice. We do this in two major ways. First, we derive performance bounds
for the case of a fixed mini-batch size, rather than requiring diverging size. Second,
we remove the need for the state-action visitation measure based gradient,
instead using the REINFORCE gradient estimator.
It is nontrivial to go from a diverging mini-batch size to a fixed one.
In fact, by allowing for an arbitrarily large batch size, existing works in the literature
were able to make use of IID samples to decouple the analysis into deterministic
gradient descent/ascent and error control of stochastic gradient estimations. In
contrast, with a single trajectory or a fixed batch size, such a
decoupling is no longer feasible. In addition, the state-action visitation measure
based gradient estimations are unbiased and unbounded, while REINFORCE
gradient estimations are biased and bounded. Hence a key to the analysis is to
deal with the bias while making better use of the boundedness. Our analysis not
only addresses these challenges, but also leads to convergence results in almost
sure and high probability senses, which are stronger than the expected convergence
results that dominate the literature (for vanilla policy gradient methods).
We also emphasize that the goal of this work is to provide a deeper understanding
of a widely used algorithm, REINFORCE, with little or no modifications, rather than
tweaking it to achieve near-optimal performance bounds.
Lastly, our analysis is not the complete picture and several
open questions about the performance of policy gradient methods remain.
We discuss these issues in the conclusion.

\subsection{Contribution} 
Our major contribution can be summarized as follows.
We establish the first set of global convergence results for the REINFORCE
algorithm. 
In particular,  
we establish an anytime sub-linear   
high probability 
regret bound as well as almost sure global convergence of the average regret
with an asymptotically sub-linear rate 
  for REINFORCE, showing that the algorithm is sample efficient (\ie, with 
  polynomial/non-exponential complexity). 
To our knowledge, these (almost sure and high probability) results are
stronger than existing 
global convergence 
results for (vanilla) policy gradient methods in the
literature. 
Moreover, our convergence results
remove the non-vanishing $\Theta(1/M^p)$ term (with $M>0$ being the mini-batch size
of the trajectories and $p>0$ being some constant exponent)
and hence show for the first time that policy gradient
estimations with a single or finite number of trajectories also enjoy global
convergence properties. 
Finally, the widely-used REINFORCE gradient
estimation procedure is studied, as opposed to the state-action visitation measure
based estimators typically studied in the literature but rarely used in
practice.

\section{Problem setting and preliminaries} \label{setting_prelim}
Below we begin with our problem setting and some preliminaries on MDPs and
policy optimization. For brevity we restrict ourselves to the stationary
infinite-horizon discounted setting. We briefly  
discuss potential extensions beyond this setting in \S \ref{conclusion}.

\subsection{Problem setting}
We consider a finite MDP 
$\mathcal{M}$, which is
characterized by a finite state space $\mathcal{S}=\{1,\dots,S\}$, a finite action
space $\mathcal{A}=\{1,\dots,A\}$, a transition probability $p$ (with $p(s'|s,a)$
being the probability of transitioning to state $s'$ given the current state $s$
and action $a$), a reward function $r$ (with $r(s,a)$ being the instantaneous
reward when taking action $a$ at state $s$), a discount factor $\gamma\in[0,1)$
and an initial state distribution $\rho\in\Delta(\mathcal{S})$.  Here $\Delta(\mathcal{X})$
denotes the probability simplex over a finite set $\mathcal{X}$.
A (stationary, stochastic) policy $\pi$  is a mapping from $\mathcal{S}$ to
$\Delta(\mathcal{A})$. We will use $\pi(a|s)$, $\pi(s,a)$ or $\pi_{s,a}$ alternatively
to denote the probability of taking action $a$ at state $s$ following policy $\pi$.
The policy $\pi$ can also be viewed as an $SA$ dimensional vector in
\BEQ\label{pi}
\begin{split}
\Pi=\Big\{\pi\in \reals^{SA}\,\Big|\,&\sum\nolimits_{a=1}^{A}\pi_{s,a}=1\,
(\forall s\in \mathcal{S}),\,\\
&\pi_{s,a}\geq 0\, (\forall s\in \mathcal{S},\,a\in \mathcal{A})\Big\}.
\end{split}
\EEQ
Notice that here we use the double indices $s$ and $a$ for notational convenience.
We use $\pi(s,\cdot)\in\reals^A$ to denote the sub-vector $(\pi(s,1),\dots,\pi(s,A))$.
We also assume that $r(s,a)$ is deterministic for any $s\in \mathcal{S}$ and $a\in
\mathcal{A}$ for simplicity, although our results hold for any $r$ with an
almost sure uniform bound. Here $r$ can be similarly viewed as an $SA$-dimensional vector.
Without loss of generality, we assume that $r(s,a)\in[0,1]$ for all $s\in\mathcal{S}$ and
$a\in\mathcal{A}$, which is a common assumption
\cite{ucrl2, kakade_2019, softmax_PG, Q-learning-rates, jin2018q}. We also assume
that $\rho$ is component-wise positive, as is assumed in \cite{russo_PG}.

Given a policy $\pi\in\Pi$, the expected cumulative reward  of the MDP is defined as
\BEQ\label{accum_reward}
F(\pi)=\Expect\sum\nolimits_{t=0}^{\infty}\gamma^tr(s_t,a_t),
\EEQ
where $s_0\sim\rho$, $a_t\sim\pi(\cdot|s_t),~s_{t+1}\sim p(\cdot|s_t,a_t),\,\forall t\geq 0$, 
and the goal is to find a policy $\pi$ which solves the following optimization problem:
\BEQ\label{opt}
\begin{array}{ll}
\text{maximize}_{\pi\in \Pi}&F(\pi).
\end{array}
\EEQ
Any policy $\pi^\star\in\argmax_{\pi\in\Pi}F(\pi)$ is said to be optimal, and the
corresponding optimal objective value is denoted as $F^\star=F(\pi^\star)$.
Note that in
the literature, $F(\pi)$ is also commonly written as $V_{\rho}^{\pi}$ and referred to as
the value function. Here we hide the dependency on $\rho$ as it is fixed throughout the paper.

\subsection{Vanilla policy gradient method and REINFORCE algorithm}
When the transition probability $p$ and reward $r$ are fully known, problem \eqref{opt}
reduces to solving an MDP, in which case various classical algorithms are available,
including value iteration and policy iteration
\cite{bertsekas1995dynamic}. 
In this paper, we consider the episodic reinforcement
learning setting in which the agent accesses $p$ and $r$ by interacting with the environment over 
successive episodes, \ie, the agent accesses the environment in the form of a
$\rho$-restart model \cite{shani2019adaptive}, which is
commonly adopted in the policy gradient literature
\cite{kakade2003sample}. In addition, we focus on the REINFORCE algorithm, a
representative policy gradient method. \vspace{-0.1cm}

\paragraph{Policy parametrization and surrogate objectives.}
Here we consider parametrizing the policy
with parameter $\theta \in \Theta$, \ie,
$\pi_{\theta}:\Theta\rightarrow\Pi$,
and take 
the following (regularized) optimization problem as an approximation to \eqref{opt}:
\BEQ\label{opt_theta}
\begin{array}{ll}
\text{maximize}_{\theta\in \Theta} &L_{\lambda}({\theta})=F(\pi_{\theta})+\lambda R(\theta),
\end{array}
\EEQ
where $\lambda\geq0$ and $R:\Theta\rightarrow\reals$ is a differentiable
regularization term that improves convergence, to be specified later. Although
our ultimate goal is still to solve the original problem \eqref{opt} this
regularized optimization problem is a useful surrogate and our approach will be to
tackle problem \eqref{opt_theta} with progressively smaller $\lambda$
regularization penalties, thereby converging to solving the actual problem we
care about.

\paragraph{Policy gradient method.}
In each episode $n$, the policy gradient method directly performs an online stochastic
gradient ascent update on a surrogate objective $L_{\lambda^n}({\theta})$, \ie,
\BEQ
\theta^{n+1}=\theta^n+\alpha^n\widehat{\nabla}_{\theta} L_{\lambda^n}({\theta^n}),
\EEQ
where
$\alpha^n$ is the step-size and $\lambda^n$ is the regularization parameter.
Here the stochastic gradient $\widehat{\nabla}_{\theta} L_{\lambda^n}({\theta^n})$ is
computed by sampling a \textit{single} trajectory $\tau^n$ following policy
$\pi_{\theta^n}$ from $\mathcal{M}$ with the initial state distribution $\rho$.
Here $\tau^n=(s_0^n,a_0^n,r_0^n,s_1^n,a_1^n,r_1^n,\dots,s_{H^n}^n,a_{H^n}^n, r_{H^n}^n)$,
where $H^n$ is a finite (and potentially random) stopping time of the trajectory
(to be specified below), $s_0^n\sim \rho$, $a_t^n\sim\pi_{\theta^n}(\cdot|s_t^n)$,
$s_{t+1}^n\sim p(\cdot|s_t^n,a_t^n)$ and $r_t^n=r(s_t^n,a_t^n)$ for all $t= 0,\dots,H^n$.
We summarize the generic policy gradient method (with single trajectory gradient estimates)
in Algorithm \ref{PG_prototype}. An extension to mini-batch scenarios will be discussed in
\S\ref{mini-batch-reinforce}. As is always (implicitly) assumed in the literature of episodic 
reinforcement learning (\eg, \cf \cite{marbach2001simulation}), given the current policy,
we assume that 
the sampled trajectory is conditionally independent of all previous policies and trajectories.

{\linespread{1.29}
\begin{algorithm}[h]
\caption{\textbf{Policy Gradient Method with Single Trajectory Estimates}}
\label{PG_prototype}
\begin{algorithmic}[1]
\STATE {\bfseries Input:} initial parameter $\theta^0$,
step-sizes $\alpha^n$ and regularization parameters $\lambda^n$ ($n\geq 0$).
\FOR{$n=0, 1, \dots$}
\STATE Choose $H^n$, sample trajectory $\tau^n$
from $\mathcal{M}$  following policy $\pi_{\theta^n}$, and
 compute an approximate gradient $\widehat{\nabla}_{\theta} L_{\lambda^{n}}({\theta^n})$
 of $L_{\lambda^n}$ using trajectory $\tau^n$.
\STATE Update $\theta^{n+1}=\theta^n+\alpha^n\widehat{\nabla}_{\theta}
L_{\lambda^{n}}({\theta^n})$.
\ENDFOR
\end{algorithmic}
\end{algorithm}
}\vspace{-0.3cm}

\paragraph{REINFORCE algorithm.}
There are several ways of choosing the stochastic gradient operator
$\widehat{\nabla}_{\theta}$ in the policy gradient method, and the
well-known REINFORCE algorithm \cite{REINFORCE} corresponds to a specific
family of estimators based on the policy gradient theorem \cite{sutton2000policy}
(\cf \S\ref{assump_examples}). Other common alternatives include zeroth
order/random search \cite{PG_LQR,malik2018derivative} and actor-critic
\cite{actor-critic} approximations. One may also choose to parametrize the policy
as a mapping from the parameter space to a specific action, which would then
result in deterministic policy gradient approximations \cite{DPG}.

Although our main goal is to study the REINFORCE algorithm, our analysis
indeed holds for rather generic stochastic gradient  estimates. 
In the next section, we introduce
the (mild) assumptions needed for our convergence analysis and the detailed
gradient estimation procedures in the REINFORCE algorithm, and then verify
that the assumptions do hold for these gradient estimations.

\subsection{Phased learning and performance criteria}
\label{phase_learn_perform_criteria}
\paragraph{Phased learning.} To facilitate the exposition below,
we divide the optimization in Algorithm \ref{PG_prototype} into successive
phases $l=0,1,\dots$, each with length $T_l>0$. We then fix the regularization
coefficient $\lambda_l$ within each phase $l\geq 0$. In addition, a
post-processing step is enforced at the end of each phase to produce the
initialization of the next phase.
The resulting algorithm is described in Algorithm \ref{PG_phased}.
Here the trajectory is denoted as 
$\tau^{l,k}=(s_0^{l,k},a_0^{l,k},r_0^{l,k},\dots,
s_{H^{l,k}}^{l,k},a_{H^{k,l}}^{l,k}, r_{H^{l,k}}^{l,k})$, and we will refer to
$\theta^{l,k}$ as the $(l,k)$-th iterate hereafter.
The post-processing function is required to guarantee that the resulting
policy $\pi_{\theta}$ is lower bounded by a pre-specified
tolerance $\epsilon_{\rm pp}\in(0,1/A]$ to ensure that the regularization
is bounded
(\cf Algorithm \ref{postprocess} for a formal description and \S \ref{assumptions}
for an example realization).

Note that here the $k$-th episode in phase $l$ corresponds to the
$n$-th episode in the original indexing with $n=\sum_{j=0}^{l-1}T_j+k$.
For notational compactness below, for $\mathcal{T}=\{T_j\}_{j=0}^{\infty}$,
we define $B_{\mathcal{T}}:{\bf Z}_+\times {\bf Z}_+\rightarrow {\bf Z}_+$,
where $B_{\mathcal{T}}(l,k)=\sum_{j=0}^{l-1}T_j+k$ maps the double
index $(l,k)$ to the  corresponding original episode number, with
$\dom B_{\mathcal{T}}=\{(l,k)\in{\bf Z}_+\times {\bf Z}_+\,|\,0\leq k\leq T_l-1\}$.
The mapping $B_{\mathcal{T}}$ is a bijection and we denote its inverse
by $G_{\mathcal{T}}$.

{\linespread{1.29}
\begin{algorithm}[h]
\caption{\textbf{Phased Policy Gradient Method}}
\label{PG_phased}
\begin{algorithmic}[1]
\STATE {\bfseries Input:} initial parameter $\tilde{\theta}^{0,0}$,
step-sizes $\alpha^{l,k}$, regularization parameters $\lambda^l$, phase lengths
$T_l$ ($l,\,k\geq 0$) and post-processing tolerance $\epsilon_{\rm pp}\in(0,1/A]$.
\STATE Set $\theta^{0,0}=\texttt{PostProcess}(\tilde{\theta}^{0,0},\epsilon_{\rm pp})$.
\FOR{phase $l=0, 1, 2,\dots$}
\FOR{episode $k=0,1,\dots,T_l-1$}
\STATE Choose $H^{l,k}$, sample trajectory $\tau^{l,k}$
from $\mathcal{M}$  following policy $\pi_{\theta^{l,k}}$, and
 compute an approximate gradient $\widehat{\nabla}_{\theta}
 L_{\lambda^{l}}({\theta^{l,k}})$
 of $L_{\lambda^l}$ using trajectory $\tau^{l,k}$.
\STATE Update $\theta^{l,k+1}=\theta^{l,k}+\alpha^{l,k}\widehat{\nabla}_{\theta}
L_{\lambda^{l}}({\theta^{l,k}})$.
\ENDFOR
\STATE Set $\theta^{l+1,0}=\texttt{PostProcess}(\theta^{l,T_l},\epsilon_{\rm pp})$.
\ENDFOR
\end{algorithmic}
\end{algorithm}
}

{\linespread{1.29}
\begin{algorithm}[h]
\caption{$\texttt{PostProcess}(\theta,\epsilon_{\rm pp})$}
\label{postprocess}
\begin{algorithmic}
\STATE {\bfseries Input:} $\epsilon_{\rm pp}\in(0,1/A]$, $\theta \in \Theta$.
\STATE \textbf{Return} $\theta'$ (near $\theta$) such that
$\pi_{\theta'}(s,a)\geq \epsilon_{\rm pp}$ for each 
$s,\,a\in\mathcal{S}\times\mathcal{A}$.
\end{algorithmic}
\end{algorithm}
}

\paragraph{Performance criteria.} The criterion we adopt to evaluate the performance of Algorithm
\ref{PG_phased} is \emph{regret}. 
For any $N\geq 0$, the regret up to episode $N$ is defined as the cumulative
sub-optimality of the policy over the $N$ episodes.
Formally, we define
\BEQ\label{regret_def}
\begin{split}
\mathbf{regret}(N)&=\sum\nolimits_{\{(l,k)|B_{\mathcal{T}}(l,k)\leq N\}}
F^\star-\hat{F}^{l,k}(\pi_{\theta^{l,k}}).
\end{split}
\EEQ
Here the summation is over all $(l,k)$-th iterates whose corresponding
original episode numbers are smaller or equal to $N$, and 
$\hat{F}^{l,k}(\pi_{\theta^{l,k}})=
\Expect_{l,k}\sum\nolimits_{t=0}^{H^{l,k}}\gamma^tr(s_t^{l,k},a_t^{l,k})$, 
where $s_0\sim\rho$, 
$a_t^{l,k}\sim\pi_{\theta^{l,k}}(\cdot|s_t^{l,k})$, 
$s_{t+1}^{l,k}\sim p(\cdot|s_t^{l,k},a_t^{l,k})$, $\forall t\geq 0$, 
and 
$\Expect_{l,k}$ denotes the conditional expectation given the $(l,k)$-th iteration
$\theta^{l,k}$. Notice that the regret
defined above takes into account the fact that 
the trajectories are
stopped/truncated
to have finite horizons $H^{l,k}$, which characterizes the 
actual loss caused by sampling the trajectories
in line 5 of Algorithm \ref{PG_phased}. A similar regret definition for the episodic (discounted)
reinforcement learning setting considered here is adopted in \cite{fu2020single}.  
We remark that all our regret bounds remain
correct up to lower order terms when we replace $\hat{F}^{l,k}$ with $F$ or an
expectation-free version.

Similarly, we also define the single phase version of regret as follows.
The regret up to episode $K\in\{0,\dots,T_l-1\}$ in phase $l$ is defined as
\BEQ\label{regret_def_phase}
\mathbf{regret}_l(K)=\sum\nolimits_{k=0}^{K} F^\star-\hat{F}^{l,k}(\pi_{\theta^{l,k}}).
\EEQ

Notice that \eqref{regret_def} and \eqref{regret_def_phase} are connected via
\BEQ\label{regret_def_connection}
\mathbf{regret}(N) = \sum_{l=0}^{l_N-1}\mathbf{regret}_l(T_l-1)+\mathbf{regret}_{l_N}(k_N),
\EEQ
where $(l_N,k_N)=G_{\mathcal{T}}(N)$.

We provide high probability regret bounds
in \S\ref{main_conv}.   
We remark that a regret
bound of the form $\mathbf{regret}(N)/(N+1)\leq R$ (for some $R>0$) immediately
implies that $\min_{l,k:\,B_{\mathcal{T}}(l,k)\leq N}F^\star-F(\pi_{\theta^{l,k}})\leq R$,
where the latter is also a commonly adopted performance criteria in the literature
\cite{kakade_2019, neural_PG}.

\section{Assumptions and  REINFORCE gradients}\label{assump_examples}
\subsection{Assumptions}\label{assumptions}
Here we list a few fundamental assumptions that we require for our analysis. 
\begin{assumption}[Setting]\label{softmax-reg}
The regularization term is a log-barrier, \ie, 
\[
R(\theta)=\tfrac{1}{SA}\sum\nolimits_{s\in\mathcal{S},
a\in\mathcal{A}}\log(\pi_\theta(s,a)),
\]  and the policy is parametrized to be a
soft-max, \ie,
$\pi_{\theta}(s,a)=\frac{\exp(\theta_{s,a})}{\sum_{a'\in\mathcal{A}}\exp(\theta_{s,a'})}$,
 with $\Theta=\reals^{SA}$.
\end{assumption}
The
first assumption concerns the form of the policy parameterization and the
regularization. 
Notice that the regularization term here can also be seen as a relative entropy/KL 
regularization (with a uniform distribution policy reference) \cite{kakade_2019}. 
Such kind of regularization 
terms are also widely adopted in practice (although typically with variations) 
\cite{REPS, schulman2017equivalence}.

With Assumption \ref{softmax-reg},  the post-processing function in Algorithm
\ref{postprocess} can be for example realized by first calculating
$\hat{\pi}=\epsilon_{\rm pp} {\bf 1}+ (1-A\epsilon_{\rm pp})\pi_{\theta}$, and then
return $\theta'$ with $\theta'_{s,a}=\log\hat{\pi}_{s,a}+c_s$. Here $\bf 1$ is an
all-one vector and $c_s\in\reals$ ($s=1,\dots,S$) are arbitrary real numbers.

\begin{assumption}[Policy gradient estimator]\label{unbiased-boundedvar}
There exist constants $C,\,C_1,\,C_2,\,M_1,\,M_2>0$, such that for all $l$,
$k\geq0$, we have
$\|\widehat{\nabla}_{\theta} L_{\lambda^l}(\theta^{l,k})\|_2\leq C_1$ almost surely and that
\begin{equation*}
\small
\begin{split}
&\nabla_{\theta} L_{\lambda^l}(\theta^{l,k})^T\Expect_{l,k}\widehat{\nabla}_{\theta}
L_{\lambda^l}(\theta^{l,k})\geq C_2\|\nabla_{\theta}
L_{\lambda^l}(\theta^{l,k})\|_2^2-\delta_{l,k},\\
&\Expect_{l,k}\|\widehat{\nabla}_{\theta} L_{\lambda^k}(\theta^{l,k})\|_2^2\leq
M_1+M_2\|\nabla_{\theta} L_{\lambda^l}(\theta^{l,k})\|_2^2,
\end{split}
\end{equation*}
where $\sum_{k=0}^{T_l-1}\delta_{l,k}^2\leq C$, $\forall$ $l\geq0$. In addition,
$H^{l,k}\geq \log_{1/\gamma}(k+1)$, $\forall$ $l,\,k\geq 0$.
\end{assumption}
The second assumption requires that the gradient estimates are almost surely
bounded, nearly unbiased and satisfy a bounded second-order moment growth
condition. This is a slight generalization of standard assumptions in the stochastic
gradient descent literature \cite{sgd_bottou}. 
Additionally, we also require that 
the trajectory lengths $H^{l,k}$ 
are at least logarithmically growing in $k$ to control the loss of rewards
due to truncation. 
For notational simplicity, hereafter we omit to mention 
the trajectory sampling
(\ie, $s_0\sim\rho,
~a_t^{l,k}\sim\pi_{\theta^{l,k}}(\cdot|s_t^{l,k}),~s_{t+1}^{l,k}\sim p(\cdot|s_t^{l,k},a_t^{l,k}),\,\forall t\geq 0$)
when we write down $\Expect_{l,k}$. 

Notice that Assumption \ref{unbiased-boundedvar} immediately holds if
$\widehat{\nabla}_{\theta} L_{\lambda^l}(\theta^{l,k})$ is unbiased and has a
bounded second-order moment.
We have implicitly assumed
that $L_{\lambda}$ is differentiable, which we can do due to the following
lemma:
\begin{proposition}[\mbox{\cite[Lemma E.4]{kakade_2019}}]\label{Lsmooth}
Under Assumption \ref{softmax-reg},
$L_{\lambda}$ is strongly smooth with parameter $\beta_{\lambda}=
\frac{8}{(1-\gamma)^3}+\frac{2\lambda}{S}$, \ie, $\|\nabla_{\theta}
L_{\lambda}(\theta)-\nabla_{\theta}L_{\lambda}(\theta')\|_2\leq
\beta_\lambda\|\theta-\theta'\|_2$ for any $\theta,\,\theta'\in\reals^{SA}$.
\end{proposition}

\subsection{REINFORCE gradient estimations}\label{reinforce_grad_est}
Now we introduce REINFORCE gradient estimation with baselines,
and specify the
hyper-parameters under which the technical Assumption
\ref{unbiased-boundedvar} holds, when operating under the setting Assumption
\ref{softmax-reg}.

REINFORCE gradient estimation with log-regularization takes the following form:
\BEQ\label{reinforce_grad_base}
\small
\begin{split}
\widehat{\nabla}_{\theta} L_{\lambda^l}(\theta^{l,k})=&\sum\nolimits_{t=0}^{\lfloor
\beta H^{l,k}\rfloor}\gamma^t(\widehat{Q}^{l,k}(s_t^{l,k},a_t^{l,k})-b(s_t^{l,k}))\\
&\qquad \times
\nabla_{\theta}\log\pi_{\theta^{l,k}}(a_t^{l,k}|s_t^{l,k})\\
&+\tfrac{\lambda^l}{SA}\sum\nolimits_{s\in\mathcal{S},a\in\mathcal{A}}\nabla_{\theta}
\log\pi_{\theta^{l,k}}(a|s),
\end{split}
\EEQ
where $\beta\in(0,1)$,
$\widehat{Q}^{l,k}(s_t^{l,k},a_t^{l,k})=\sum_{t'=t}^{H^{l,k}}\gamma^{t'-t}r_{t'}^{l,k}$,
and the second term above corresponds to the gradient of the regularization
$R(\theta)$. Notice that here the outer summation is only up to
$\lfloor \beta H^{l,k}\rfloor$, which ensures that $\hat{Q}^{l,k}(s_t^{l,k},a_t^{l,k})$
is sufficiently accurate.
Here $b:\mathcal{S}\rightarrow\reals$ is called the baseline, and is required to
be independent of the trajectory $\tau^{l,k}$ \cite[\S 4.1]{agarwal2019reinforcement}.
The purpose of subtracting $b$ from the approximate $Q$-values is to (potentially)
reduce the variance of the ``plain'' REINFORCE gradient estimation,
which corresponds to the case when $b= 0$.

With this we have the following result, the proof of which can be found in the Appendix.
\begin{lemma}\label{verify_basic_reinforce_base}
Suppose that Assumption \ref{softmax-reg} holds, $\beta\in(0,1)$,
and that for all $l$, $k\geq 0$,
$\lambda^l\leq \bar{\lambda}$, 
\BEQ\label{Hlk_bound}
H^{l,k}\geq \tfrac{2\log_{1/\gamma}\left(\frac{8(k+1)}{(1-\gamma)^{3}}\right)}
{3\min\{\beta,1-\beta\}}(=\Theta(\log(k+1))).
\EEQ
Assume in addition that $|b(s)|\leq B$ for any $s\in\mathcal{S}$, 
where $B>0$ is a constant. Then for the gradient estimation
\eqref{reinforce_grad_base}, Assumption \ref{unbiased-boundedvar} holds with
\[
\begin{split}
&\text{$C=16\Big(\tfrac{1}{(1-\gamma)^2}+\bar{\lambda}\Big)^2$, \quad
$C_1=\tfrac{2(1+B(1-\gamma))}{(1-\gamma)^2}+2\bar{\lambda}$,} \\
&\text{$C_2=1$,\quad
$M_1=\tfrac{32}{(1-\gamma)^4}+\bar{V}_b$, \quad $M_2=2$.}
\end{split}
 \]
 and $\delta_{l,k}=\Big(\tfrac{2}{(1-\gamma)^2}+2\bar{\lambda}\Big)(k+1)^{-\frac{2}{3}}$,
 $\forall$ $l$, $k\geq 0$.
 Here $\bar{V}_b\in\Big[0,4\Big(\tfrac{1+B(1-\gamma)}{(1-\gamma)^2}+
 \bar{\lambda}\Big)^2\Big]$ is the uniform upper bound on variances of
 policy gradient estimations with form \eqref{reinforce_grad_base}.
\end{lemma}
This result extends without modification to non-stationary baselines
$b_t^{l,k}$, as long as each $b_t^{l,k}$ is independent of trajectory
$\tau^{l,k}$ and $|b_t^{l,k}(s)|\leq B$ for any $t,\,l,\,k\geq 0$.
Note that the explicit upper bound on $\bar{V}_b$ is pessimistic, and in
practice $\bar{V}_b$ is usually much smaller than $\bar{V}_0$ with appropriate
choices of baselines (\eg, the adaptive reinforcement baselines
\cite{REINFORCE,pge_var}), although the latter has a smaller upper
bound as stated in Lemma \ref{verify_basic_reinforce_base}.

\section{Main convergence results}\label{main_conv}

\subsection{Preliminary tools}
We first present some preliminary tools 
for our analysis. \vspace{-0.21cm}

\paragraph{Non-convexity and control of ``bad'' episodes.} 
One of the key difficulties
in applying policy gradient methods to solve an MDP problem towards global optimality
 is that problem 
\eqref{opt} is in general non-convex \cite{kakade_2019}. Fortunately, we have 
the following result, which connects
the gradient of the surrogate objective $L_{\lambda}$ with the global optimality gap
of the original optimization problem \eqref{opt}. 
\begin{proposition}[\mbox{\cite[Theorem 5.3]{kakade_2019}}]\label{kakade_prop}
Under Assumption \ref{softmax-reg}, for any $\epsilon>0$, 
suppose that we have $\|\nabla_{\theta}
L_{\lambda}(\theta)\|_2\leq \epsilon$ and that $\epsilon\leq \lambda/(2SA)$.
Then
$F^\star-F(\pi_{\theta})\leq \frac{2\lambda}{1-\gamma}
\left\|\frac{d_{\rho}^{\pi^\star}}{\rho}\right\|_{\infty}$.
\end{proposition}
Here for any policy 
$\pi\in\Pi$, $d_{\rho}^{\pi}=(1-\gamma)\sum_{t=0}^{\infty}\gamma^t\prob_{\pi}(s_t
=s|s_0\sim\rho)$ is the discounted state visitation distribution, where
$\prob_{\pi}(s_t=s|s_0\sim\rho)$ is the probability of arriving at $s$ in step $t$
starting from $s_0\sim\rho$ following policy $\pi$ in $\mathcal{M}$. In addition,
 the division in $d_{\rho}^{\pi^\star}/\rho$ is component-wise. 

Now motivated by Proposition \ref{kakade_prop}, when analyzing the regret up to
episode $K$ in phase $l$, we define the following set of ``bad episodes'':
\begin{equation*}
I^+=\{k\in\{0,\dots,K\}\,|\,\|\nabla_{\theta}L_{\lambda^l}(\theta^{l,k})\|_2\geq 
\lambda^l/(2SA)\}.
\end{equation*}
Then one can show that for any $\epsilon>0$,  if we choose 
$\lambda^l=
\frac{\epsilon(1-\gamma)}{2}$,
  we have that 
 $F^\star-F(\pi_{\theta^{l,k}})\leq\|d_{\rho}^{\pi^\star}/\rho\|_\infty\epsilon$ for any 
 $k\in\{0,\dots,K\}\backslash I^+$, while 
 $F^\star-F(\pi_{\theta^{l,k}})\leq 1/(1-\gamma)$ holds trivially for 
 $k\in I^+$ due to the assumption that the rewards are between $0$ and $1$. 
 We then establish a sub-linear (in $K$) bound the size of $I^+$, which serves
 as the key stepping stone for the overall sub-linear regret bound.  
 The details of these arguments can be found in the Appendix.

\paragraph{Doubling trick.} 
The second tool is a classical
doubling trick that is commonly adopted in the design of online learning
algorithms \cite{besson2018doubling,basei2020linear}, which can be used
to stitch together the regret over multiple learning phases in Algorithm
\ref{PG_phased}.

Notice that Proposition \ref{kakade_prop} suggests that for any 
pre-specified tolerance $\epsilon$, one can
select $\lambda$ proportional to $\epsilon$ and then run (stochastic) gradient
ascent to drive $F^\star-F(\pi_{\theta})$ below the tolerance. To obtain the
eventual convergence and regret bound in the long run we apply the
doubling trick, which specifies a growing phase length sequence
with $T_{l+1}\approx 2T_l$ in Algorithm \ref{PG_phased} and a suitably decaying
sequence of regularization parameters $\{\lambda^l\}_{l=0}^{\infty}$. 

\paragraph{From high probability to almost sure convergence.} 
The last tool is an observation that an arbitrary anytime sub-linear high probability regret bound with 
logarithmic dependency
on $1/\delta$ immediately leads to almost sure convergence of the average regret 
with a corresponding asymptotic rate. Although such an observation seems to be informally 
well-known in the theoretical computer science community, we provide a compact formal 
discussion below for self-contained-ness. 
\begin{lemma}\label{highprob2as}
Suppose that $\forall$ $\delta\in(0,1)$, with probability at least $1-\delta$, $\forall$ $N\geq 0$, 
we have 
\BEQ\label{abstract_bounds}
\small
\mathbf{regret}(N)\leq d_1(N+c)^{d_2}(\log(N/\delta))^{d_3}+d_4(\log N)^{d_5}
\EEQ
 for some constants 
$c,\,d_1,\,d_3,\,d_4,\,d_5\geq 0$ and $d_2\in[0,1)$. Then we also have
\[
\begin{split}
\prob&\left(\text{$\exists$ 
$\bar{N}\in\mathbf{Z}_+$, 
such that $\forall$ $N\geq \bar{N}$, $A_N$ holds}\right)=1,
\end{split}
\]
where the events $A_N=\{\mathbf{regret}(N)/(N+1)\leq (*)\}$, and  
\[
(*)=d_1 N^{-(1-d_2)}\left(1+\tfrac{c}{N}\right)^{d_2}(3\log N)^{d_3}
+\tfrac{d_4(\log N)^{d_5}}{N}.
\]
To put it another way, we have 
\[
\lim_{N\rightarrow\infty}\mathbf{regret}(N)/(N+1)=0\quad \text{almost surely}
\]
with an asymptotic rate of $(*)$. 
\end{lemma}
Notice that here we restrict the right-hand side of \eqref{abstract_bounds} to 
a rather specific form simply because our bounds below are all of this form. 
However, similar 
results hold for much more general forms of bounds.

\subsection{Regret analysis}\label{regret_analysis}

In this section, we establish the regret bound of Algorithm \ref{PG_phased},
when used with the REINFORCE gradient estimator from \S \ref{reinforce_grad_est}.
We begin by bounding the regret of a single phase and then use the doubling
trick to combine these into the overall regret bound. \vspace{-0.12cm}

\paragraph{Single phase analysis.}

We begin by 
bounding the regret defined in
\eqref{regret_def_phase} of each phase in Algorithm \ref{PG_phased}. 
Note that a single phase in Algorithm \ref{PG_phased} is exactly Algorithm
\ref{PG_prototype} terminated
 in episode $T_l$, with
$\lambda^n=\lambda^l$ for all $n\geq 0$ and $\theta^0=\theta^{l,0}$.  
Also notice that for a given phase $l\geq0$, in order for Theorem \ref{phase_regret}
below to hold, we actually only need the conditions  in
Assumption \ref{unbiased-boundedvar} to be satisfied for this specific $l$.

\begin{theorem}\label{phase_regret}
Under Assumptions \ref{softmax-reg} and \ref{unbiased-boundedvar}, for phase
$l\geq 0$ suppose that we choose
$\alpha^{l,k}=C_{l,\alpha}\frac{1}{\sqrt{k+3}\log_2(k+3)}$ for some
 $C_{l,\alpha}\in(0,C_2/(M_2\beta_{\lambda^l})]$.
Then for any $\epsilon>0$, if we choose
$\lambda^l=\frac{\epsilon(1-\gamma)}{2}$,
then 
$\forall$ $\delta\in(0,1)$, with probability at least $1-\delta$,
for any $K\in\{0,\dots,T_l-1\}$, we have
\BEQ\label{regret_phase_l}
\begin{split}
&\mathbf{regret}_l(K)\leq U_1\tfrac{\sqrt{K+1}\log_2(K+3)\sqrt{\log(2/\delta)}}{\epsilon^2}\\
&\qquad+
(K+1)\Big\|\tfrac{d_\rho^{\pi^\star}}{\rho}\Big\|_\infty\epsilon+\tfrac{2\gamma}{1-\gamma}\log (K+3).
\end{split}
\EEQ
Here the constant $U_1$ only depends on 
the underlying MDP $\mathcal{M}$, phase initialization $\theta^{l,0}$ and the constants  $C,\,C_1,\,C_2$, 
$M_1,\,C_{l,\alpha},\,\lambda^l$. 
\end{theorem}

The proof as well as  a more formal statement of Theorem \ref{phase_regret} 
with details of the constants 
(\cf Theorem \ref{phase_regret_formal}) are deferred to the Appendix. Here the
constant $\beta_{\lambda}$ is the smoothness constant from Proposition \ref{Lsmooth}.
We remark that when $\epsilon$ is fixed, 
the regret bound \eqref{regret_phase_l}
can be seen as a sub-linear (in $K$ as $K\rightarrow\infty$) regret term plus an
error term $(K+1)\epsilon+\frac{2\gamma}{1-\gamma}\log(K+3)$. Alternatively, one can interpret it as follows:
\[
\begin{split}
\dfrac{\mathbf{regret}_l(K)}{K+1}\leq&\,
U_1\dfrac{\log_2(K+3)\sqrt{\log(2/\delta)}}{\sqrt{K+1}\epsilon^2}
\\
&\,+\dfrac{2\gamma}{1-\gamma}\dfrac{\log (K+3)}{K+1}+\Big\|\tfrac{d_\rho^{\pi^\star}}{\rho}\Big\|_\infty\epsilon. 
\end{split}
\]
Namely, the average regret in episode $l$ converges to a constant multiple of the pre-specified
tolerance $\epsilon$ at a sub-linear rate (as $K\rightarrow\infty$).

\paragraph{Overall regret bound.} Now we stitch together the single phase
regret bounds established above to obtain the overall regret bound of Algorithm
\ref{PG_phased}, with the help of the doubling trick. 
This leads to the following
theorem.
\begin{theorem}[Regret for REINFORCE]\label{regret_reinforce}
Under Assumption \ref{softmax-reg},
suppose that for each $l\geq 0$, we choose
$\alpha^{l,k}=C_{l,\alpha}\frac{1}{\sqrt{k+3}\log_2(k+3)}$, with
$C_{l,\alpha}\in[1/(2\beta_{\bar{\lambda}}), 1/(2\beta_{\lambda^l})]$ and
$\bar{\lambda}=\frac{1-\gamma}{2}$, and choose
$T_l=2^l$, $\epsilon^l=T_l^{-1/6}=2^{-l/6}$,
$\lambda^l=\frac{\epsilon^l(1-\gamma)}
{2}$ and $\epsilon_{\rm pp}=1/(2A)$.
In addition, suppose that \eqref{reinforce_grad_base} is adopted to
evaluate $\widehat{\nabla}_{\theta} L_{\lambda^l}(\theta^{l,k})$,
with $\beta\in(0,1)$, $|b(s)|\leq B$ for any $s\in\mathcal{S}$ (where $B>0$ is a
constant), and
that \eqref{Hlk_bound} holds for $H^{l,k}$ for all $l,\,k\geq 0$. Then we have 
for any $\delta\in(0,1)$, with probability at least $1-\delta$, for any $N\geq 0$,
we have
\BEQ\label{regret_reinforce_simp}
\small
\mathbf{regret}(N)=O\Big(\Big(\tfrac{S^2A^2}{(1-\gamma)^{7}}+
\Big\|\tfrac{d_{\rho}^{\pi^\star}}{\rho}\Big\|_{\infty}\Big)
N^{\frac{5}{6}}(\log\tfrac{N}{\delta})^{\frac{5}{2}}\Big).
\EEQ
In addition, we have
\BEQ\label{regret_reinforce_as}
\lim_{N\rightarrow\infty}\mathbf{regret}(N)/(N+1)=0\quad \text{almost surely}
\EEQ
with asymptotic rate {\small$O\Big(\Big(\frac{S^2A^2}{(1-\gamma)^{7}}+
\Big\|\frac{d_{\rho}^{\pi^\star}}{\rho}\Big\|_{\infty}\Big)
N^{-\frac{1}{6}}(\log N)^{\frac{5}{2}}\Big)$}.
\end{theorem}

Note that the almost sure convergence \eqref{regret_reinforce_as} is immediately implied by the high probability bound
\eqref{regret_reinforce_simp} via Lemma \ref{highprob2as}. 
 Here for clarity, we have restricted the statement 
to the case when we use the REINFORCE gradient estimation from \S\ref{reinforce_grad_est}. 
A more general counterpart result can be found in Appendix \ref{overall_regret_general}, from 
which Theorem \ref{regret_reinforce} is immediately implied. 
See also Appendix \ref{formal_statement} for
a more formal statement of the regret bound (\cf Corollary \ref{regret_reinforce_formal})
 for REINFORCE
with detailed constants.

Notice that compared to the single phase regret bound in \eqref{regret_phase_l},
the overall regret bound in \eqref{regret_reinforce_simp} now gets rid of the dependency
on a pre-specified tolerance $\epsilon>0$. This should be attributed to the
adaptivity in the regularization parameter sequence.  Also notice that here we have followed 
the convention of the reinforcement learning literature to make 
all the problem dependent quantities (\eg, $\gamma,\,S,\,A$, etc.) explicit in the 
big-$O$ notation. 

One crucial difference between our regret bound and those in the
existing literature of vanilla policy gradient methods in the general MDP settings
(which are sometimes not stated in the form of regret, but can
be easily deduced from their proofs in those cases) is that the previous results either
require exact and deterministic updates or 
contain a non-vanishing $\Theta(1/M^p)$ term, with $M$ being the mini-batch size
(of the trajectories) and $p>0$ being some exponent (with a typical value of $1/2$). 
By removing such non-vanishing terms, we obtain the first sub-linear regret bound for model-free
 vanilla policy gradient methods with finite mini-batch sizes.

\section{Extension to mini-batch updates}
\label{mini-batch-reinforce}
We now consider extending our previous results to mini-batch settings,
by modifying Algorithm \ref{PG_phased} as follows. Firstly, in each inner iteration,
instead of sampling only one trajectory in line 5, we sample $M\geq 1$
independent trajectories $\tau^{l,k}_1,\dots,\tau^{l,k}_M$ from $\mathcal{M}$
following policy $\pi_{\theta^{l,k}}$ and then compute an approximate gradient
$\widehat{\nabla}_{\theta}^{(i)}L_{\lambda^l}(\theta^{l,k})$ ($i=1,\dots,M$) using each of
these $M$ trajectories. We then modify the update in line 6 as
\[\theta^{l,k+1}=\theta^{l,k}+\alpha^{l,k}\frac{1}{M}\sum_{i=1}^M\widehat{\nabla}_{\theta}^{(i)}
L_{\lambda^{l}}({\theta^{l,k}}).
\]
 See Algorithm \ref{PG_phased_batch} in Appendix
\ref{minibatch_PG_phased} for a formal description
of the modified algorithm.

\paragraph{Regret with mini-batches.} Notice that since each
inner iteration (in Algorithm \ref{PG_phased_batch}) now consists of $M$ episodes,
we need to slightly modify the definition of the regret up to episode $N$ ($N\geq 0$)
as follows:
\BEQ\label{regret_def_batch}
\small
\begin{split}
\mathbf{regret}&(N;M)=\\
&\sum\nolimits_{\left\{(l,k)|B_{\mathcal{T}}(l,k)\leq
\lfloor \frac{N}{M}\rfloor-1\right\}}
M(F^\star-\hat{F}^{l,k}(\pi_{\theta^{l,k}}))\\
&+\left(N-M\left\lfloor \frac{N}{M}\right\rfloor\right)
(F^\star-\hat{F}^{l,k}(\pi_{\theta^{l_{N,M},k_{N,M}}})),
\end{split}
\EEQ
where $(l_{N,M},k_{N,M})=G_{\mathcal{T}}(\lfloor N/M\rfloor)$ and
$\hat{F}^{l,k}(\pi_{\theta^{l,k}})$ is the same as in \eqref{regret_def}.
The above definition accounts for the fact that  each of the $M$ episodes in an inner
iteration/step $(l,k)$ corresponds to the same iterate $\theta^{l,k}$ and
hence has the same contribution to the regret. The second term on the right-hand side
accounts for the contribution of the (remaining) $N-M\lfloor N/M\rfloor$ episodes
(among a total of $M$ episodes) in inner iteration/step $(l_{N,M},k_{N,M})$.

Then 
the following regret bound can be established. 
\begin{corollary}[Regret for mini-batch REINFORCE]
\label{regret_reinforce_batch}
Under Assumption \ref{softmax-reg},
suppose that for each $l\geq 0$, we choose
$\alpha^{l,k}=C_{l,\alpha}\frac{1}{\sqrt{k+3}\log_2(k+3)}$, with
$C_{l,\alpha}\in[1/(2\beta_{\bar{\lambda}}), 1/(2\beta_{\lambda^l})]$ and
$\bar{\lambda}=\frac{1-\gamma}{2}$, and choose
$T_l=2^l$, $\epsilon^l=T_l^{-1/6}=2^{-l/6}$,
$\lambda^l=\frac{\epsilon^l(1-\gamma)}
{2}$ and $\epsilon_{\rm pp}=1/(2A)$. 
In addition, suppose that the assumptions in Lemma \ref{mini-batch-assumption2} hold
(note that Assumption \ref{softmax-reg} and $\lambda^l\leq \bar{\lambda}$
already automatically hold by the other assumptions).
Then we have 
for any $\delta\in(0,1)$, with probability at least $1-\delta$, 
jointly for all episodes $N$, we have   
(for the mini-batch Algorithm \ref{PG_phased_batch})
\[
\small
 \begin{split}
& \mathbf{regret}(N;M)=O\Big(\Big(\tfrac{S^2A^2}{(1-\gamma)^7}+\Big\|\tfrac{d_{\rho}^{\pi^\star}}{\rho}\Big\|_{\infty}\Big)\\
&\quad\times
(M^{\frac{1}{6}}+M^{-\frac{5}{6}})(N+M)^{\frac{5}{6}}(\log(N/\delta))^{\frac{5}{2}}
+\tfrac{M(\log N)^2}{1-\gamma}\Big).
\end{split}
\]
In addition, we also have 
\[
\lim_{N\rightarrow\infty}\mathbf{regret}(N;M)/(N+1)=0 \quad \text{almost surely}
\]
with an asymptotic rate of 
\[
\small
\begin{split}
&O\Big(\Big(\tfrac{S^2A^2}{(1-\gamma)^7}+\Big\|\tfrac{d_{\rho}^{\pi^\star}}{\rho}\Big\|_{\infty}\Big)\\
&\,\times
(M^{\frac{1}{6}}+M^{-\frac{5}{6}})N^{-\frac{1}{6}}\left(1+\tfrac{M}{N}\right)^{\frac{5}{6}}(\log N)^{\frac{5}{2}}
+\tfrac{M(\log N)^2}{(1-\gamma)N}\Big).
\end{split}
\]
\end{corollary}

 Again, we note that the almost sure convergence above is directly implied by the high 
 probability bound via Lemma \ref{highprob2as}. 
 The proof and a more formal statement of this corollary 
(\cf Corollary \ref{regret_reinforce_batch_formal})
 can be found in Appendix \ref{minibatch_PG_phased}. 
 In particular, when $M=1$, the bound above reduces to \eqref{regret_reinforce_simp}.
In addition, we can see that there might be a trade-off between the terms $M^{1/6}$ and
$M^{-5/6}$. The intuition behind this is a trade-off between lower
variance with larger batch sizes and more frequent updates with smaller batch sizes.

\section{Conclusion and open problems}\label{conclusion}
In this work, we establish the global convergence rates of practical policy
gradient algorithms with a fixed size mini-batch of trajectories
combined with REINFORCE gradient estimation.

Although in \S\ref{main_conv} and \S\ref{mini-batch-reinforce}, 
 we only instantiate the bounds for the REINFORCE gradient estimators, we note that
 our general results (in particular, Theorem \ref{overall_regret} in Appendix 
 \ref{overall_regret_general}) 
 can be easily applied to other gradient estimators 
 (\eg, actor-critic and state-action visitation measure based estimators) 
 as well, as long as one can verify
 the existence of the constants in Assumption \ref{unbiased-boundedvar}
 in a similar way to Lemma \ref{verify_basic_reinforce_base}. 
 In addition, one can also easily derive sample complexity results as 
 by-products of our analysis. In fact, our proof of Theorem \ref{phase_regret}
 immediately implies a $\tilde{O}(1/\epsilon^4)$ sample complexity bound 
 (for Algorithm \ref{PG_prototype} with REINFORCE gradient estimators 
 and a constant regularization parameter)
 for any pre-specified tolerance $\epsilon>0$, where we use 
$\tilde{O}$ to indicate the big-$O$ notation with logarithmic terms suppressed.  
We have 
 focused 
 only on regret in this paper mainly for clarity purposes. 
 
 It is also relatively straightforward to extend
our results to finite horizon non-stationary settings, in which the soft-max
policy parametrization will have a dimension of $SAH$ and different
policy gradient estimators can be adopted (without trajectory truncation),
with $H$ being the horizon
 of each episode.  
In this case,
it's also easy to
rewrite the regret bound as a function of the total number of time steps
$T\leq HN$, where $N$ is the total number of episodes.
Other straightforward
extensions include  
refined convergence to stationary points (in both almost sure and high
probability senses and with no requirement on large batch sizes), and
inexact convergence results when $\delta^{l,k}$ 
(\cf Assumption \ref{unbiased-boundedvar})
 is not square summable
(\eg, when $H^{l,k}$ is fixed or not growing sufficiently fast).

There are also several open problems that may be resolved by combining the
techniques introduced in this paper with existing results in the literature.
Firstly,
it would be desirable to remove the ``exploration'' assumption that the initial distribution $\rho$ is
component-wise positive. This may be achieved by combining our results with
the policy cover technique in \cite{PC-PG} or the
optimistic bonus tricks in \cite{OPPO, efroni2020optimistic}.
Secondly, the bounds in our paper are likely far
from optimal (\ie, sharp). Hence it would be desirable to either refine our analysis or
 apply our techniques to accelerated policy gradient methods (\eg, IS-MBPG
 \cite{momentum_PG}) to obtain better global convergence rates and/or last-iterate
 convergence. Thirdly, it would be
 very  interesting to see if global convergence results still hold for REINFORCE when
 the relative entropy regularization term used in this paper is replaced with the
 practically adopted entropy regularization term in the literature. The answer is affirmative
 when exact gradient estimations are available \cite{softmax_PG,cen2020fast},
 but it remains unknown how these results
 might be generalized to the stochastic settings in our paper. We conjecture that
 entropy regularization leads to better global convergence rates and can help us remove
 the necessity of the \texttt{PostProcess} steps in Algorithm \ref{PG_phased} as they
 are uniformly bounded. Finally, one may also
 consider relaxing the uniform bound assumption
  on the rewards $r$ to instead being sub-Gaussian, introducing function approximation, and
  extending our results to natural policy gradient and actor-critic methods as well as
  more modern policy gradient methods like DPG, PPO and TRPO.

\section*{Acknowledgment}
We would like to thank Anran Hu for pointing out a mistake in the proof of an early draft 
of this paper. We thank Guillermo Angeris, Shane Barratt, 
Haotian Gu, Xin Guo, Yusuke Kikuchi, Bennet Meyers, Xinyue Shen, Mahan Tajrobehkar, 
Jonathan Tuck, Jiaming Wang and Xiaoli Wei for their feedback on some preliminary 
results in this paper. We thank Junyu Zhang for several detailed and fruitful discussions 
of the draft. We also thank the anonymous (meta-)reviewers for the great comments and suggestions.
Jongho Kim is supported by Samsung Scholarship.

\newpage

\bibliography{conv_reinforce}

\newpage
\onecolumn
\appendix
\section*{Appendix}
\addcontentsline{toc}{section}{Appendix}

In this appendix, we provide detailed proofs and formal statements of the results in the
main text. For notational simplicity, we sometimes abbreviate ``almost sure'' as ``a.s.'' or 
even omit ``a.s.'' whenever it is clear from
the context. Also notice that as is always implicitly assumed in the literature of episodic 
reinforcement learning (\eg, \cf \cite{marbach2001simulation}), given the current policy,
the sampled trajectory is conditionally independent of all previous policies and trajectories. 

\paragraph{Big-$O$ notation.} 
We first clarify the precise definitions of the Big-$O$ notation used in 
our statements. 
Let $f:\mathbf{Z}_+\times \reals^d\rightarrow\reals_+$ be a function of the total number of 
episodes $N$
and all problem and algorithm dependent 
quantities (written jointly as a vector) $U\in\reals^d$. Similarly, let 
$g:\mathbf{Z}_+\times \reals^d\rightarrow\reals_+$ be a function of  $N$
and some (subset of) problem and algorithm dependent 
quantities $U_0\in\reals^{d_0}$, with $d_0\leq d$.
Then we write $f(N;U)=O(g(N;U_0))$ to indicate that there exist 
constants 
$W>0$ and $N_0\in \mathbf{Z}_+$ (independent of $N$ and $U_0$), 
such that $f(N;U)\leq Wg(N;U_0)$ for all $N\geq N_0$. 

\section{Proofs for REINFORCE gradient estimations}

\subsection{Proof of Lemma \ref{verify_basic_reinforce_base} (with $b= 0$)}
\begin{proof}
We validate the three groups conditions in Assumption \ref{unbiased-boundedvar} in
order. For the simplicity of exposition, we first restrict to the case when $b= 0$,
\ie, no baseline is incorporated. \vspace{1em}

\noindent\textbf{Gradient estimation boundedness.}
Firstly, notice that since $r(s,a)\in[0,1]$, we have
$\hat{Q}^{l,k}(s_t^{l,k},a_t^{l,k})\leq 1/(1-\gamma)$. And by the soft-max
parametrization in Assumption \ref{softmax-reg}, we have
\[
\nabla_{\theta}\log\pi_{\theta^{l,k}}(a|s)={\bf 1}_{s,a}-
\Expect_{a'\sim \pi_{\theta^{l,k}}(\cdot|s)}{\bf 1}_{s,a'},
\]
where the vector ${\bf 1}_{s,a}\in\reals^{SA}$ has all zero entries apart from the
one corresponding to the state-action pair $(s,a)$. Hence
$\|\nabla_{\theta}\log\pi_{\theta^{l,k}}(a|s)\|_2\leq 2$ for any
$(s,a)\in\mathcal{S}\times\mathcal{A}$, and we see that
\BEQ\label{Lhat_asbound}
\begin{split}
\left\|\widehat{\nabla}_{\theta} L_{\lambda^{l}}(\theta^{l,k})\right\|_2&\leq
 \dfrac{1}{1-\gamma}\sum_{t=0}^{\infty}\gamma^t\|\nabla_{\theta}
 \log\pi_{\theta^{l,k}}(a_t^{l,k}|s_t^{l,k})\|_2+\dfrac{\bar{\lambda}}{SA}
 \sum_{s\in\mathcal{S},a\in\mathcal{A}}\|\nabla_{\theta}\log\pi_{\theta^{l,k}}(a|s)\|_2\\
&\leq 2/(1-\gamma)^2+2\bar{\lambda},
 \quad \text{a.s.}
 \end{split}
\EEQ
Hence we can take $C_1=2/(1-\gamma)^2+2\bar{\lambda}$. \vspace{1em}

\noindent\textbf{Validation of nearly unbiasedness.}
Secondly, notice that
\[
\begin{split}
\Expect_{l,k}\widehat{\nabla}_{\theta} L_{\lambda^{l}}(\theta^{l,k})
=&\,\Expect_{l,k}\left(\sum_{t=0}^{\lfloor \beta H^{l,k}\rfloor}\gamma^t
\Expect_{l,k}\left(\hat{Q}^{l,k}(s_t^{l,k},a_t^{l,k})\Big|s_t^{l,k},a_t^{l,k}\right)\nabla_{\theta}
\log\pi_{\theta^{l,k}}(a_t^{l,k}|s_t^{l,k})\right)\\
&+\dfrac{\lambda^{l}}{SA}
\sum\nolimits_{s\in\mathcal{S},a\in\mathcal{A}}\nabla_{\theta}\log\pi_{\theta^{l,k}}(a|s)=\,J_1+J_2+J_3,
\end{split}
\]
where 
\[
\begin{split}
J_1=&\, \Expect_{l,k}\left(\sum_{t=0}^{\infty}\gamma^t\Expect_{l,k}
\left(\sum\nolimits_{t'=t}^{\infty}\gamma^{t'-t}r_{t'}^{l,k}\Big|s_t^{l,k},a_t^{l,k}\right)
\nabla_{\theta}\log\pi_{\theta^{l,k}}(a_t^{l,k}|s_t^{l,k})\right)\\
&+
\dfrac{\lambda^{l}}{SA}\sum\nolimits_{s\in\mathcal{S},a\in\mathcal{A}}
\nabla_{\theta}\log\pi_{\theta^{l,k}}(a|s),
\end{split}
\]
\[
\begin{split}
J_2&=-\Expect_{l,k}\left(\sum_{t=\lfloor \beta H^{l,k}\rfloor+1}^{\infty}
\gamma^t\Expect_{l,k}\left(\sum\nolimits_{t'=t}^{\infty}
\gamma^{t'-t}r_{t'}^{l,k}\Big|s_t^{l,k},a_t^{l,k}\right)\nabla_{\theta}
\log\pi_{\theta^{l,k}}(a_t^{l,k}|s_t^{l,k})\right),\\
J_3&=-\Expect_{l,k}\left(\sum_{t=0}^{\lfloor \beta
H^{l,k}\rfloor}\gamma^t\Expect_{l,k}\left(\sum\nolimits_{t'=H^{l,k}+1}^{\infty}
\gamma^{t'-t}r_{t'}^{l,k}\Big|s_t^{l,k},a_t^{l,k}\right)\nabla_{\theta}
\log\pi_{\theta^{l,k}}(a_t^{l,k}|s_t^{l,k})\right).
\end{split}
\]

By \cite[Theorem 4.6]{agarwal2019reinforcement}, we have
\BEQ\label{pg_lambda}
\begin{split}
J_1 = \Expect_{l,k}&\left(\sum_{t=0}^{\infty}
\gamma^tQ^{\pi_{\theta^{l,k}}}(s_t^{l,k},a_t^{l,k})\nabla_{\theta}
\log\pi_{\theta^{l,k}}(a_t^{l,k}|s_t^{l,k})\right)\\
&+\dfrac{\lambda^{l}}{SA}
\sum_{s\in\mathcal{S},a\in\mathcal{A}}\nabla_{\theta}\log\pi_{\theta^k}(a|s)\\
=\nabla_{\theta}&L_{\lambda^{l}}(\theta^{l,k}).
\end{split}
\EEQ
Here for any $\pi\in\Pi$, 
\[
Q^\pi(s,a)=\Expect\left(\sum\nolimits_{t=0}^{\infty}\gamma^tr(s_t,a_t)
\Big|s_0=s,\,a_0=a\right),
\]
with $\,a_t\sim\pi(s_t,\cdot),\,s_{t+1}\sim p(\cdot|s_t,a_t),\,\forall t>0$. 

And since $r(s,a)\in[0,1]$, we have
\begin{align*}
\|J_2\|_2&\leq \dfrac{1}{1-\gamma}\sum_{t=\lfloor \beta H^{l,k}\rfloor+1}^{\infty}
\gamma^t\|\nabla_{\theta}\log\pi_{\theta^{l,k}}(a_t^{l,k}|s_t^{l,k})\|_2\\
&\leq 2\gamma^{\beta H^{l,k}}/(1-\gamma)^2,
\end{align*}
and similarly
\[
\begin{split}
\|J_3\|_2&\leq \sum_{t=0}^{\lfloor \beta H^{l,k}\rfloor}\gamma^t
\sum_{t'=H^{l,k}+1}^{\infty}\gamma^{t'-t}\|\nabla_{\theta}
\log\pi_{\theta^k}(a_t^{l,k}|s_t^{l,k})\|_2\\
&\leq \sum_{t=0}^{\lfloor \beta H^{l,k}\rfloor}\gamma^t\times 2
\gamma^{(1-\beta) H^{l,k}}/(1-\gamma)\\
&\leq 2\gamma^{(1-\beta)
H^{l,k}}/(1-\gamma)^2.
\end{split}
\]
Hence for any $\eta_0>0$, by taking
\[
H^{l,k}\geq \frac{1+\eta_0}
{(2+\eta_0)\min\{\beta,1-\beta\}}\log_{1/\gamma}
\left(\frac{4^{(2+\eta_0)/(1+\eta_0)}(k+1)}{(1-\gamma)^{(4+2\eta_0)/
(1+\eta_0)}}\right)(=\Theta(\log (k+1))),
\]
we have $H^{l,k}\geq \log_{1/\gamma}(k+1)$, and that
\BEQ\label{Eklambda_error}
\left\|\Expect_{l,k}\widehat{\nabla}_{\theta} L_{\lambda^l}
(\theta^{l,k})-\nabla_{\theta}L_{\lambda^l}(\theta^{l,k})\right\|_2
\leq \dfrac{4\gamma^{\min\{\beta,1-\beta\} H^{l,k}}}{(1-\gamma)^2}
\leq (k+1)^{-\frac{1+\eta_0}{2+\eta_0}},
\EEQ
which implies that
\[
\begin{split}
\nabla_{\theta}&L_{\lambda^l}(\theta^{l,k})^T
\Expect_{l,k}\widehat{\nabla}_{\theta} L_{\lambda^l}(\theta^{l,k})\\
&=\|\nabla_{\theta}L_{\lambda^l}(\theta^{l,k})\|_2^2+
\left(\Expect_{l,k}\widehat{\nabla}_{\theta} L_{\lambda^l}
(\theta^{l,k})-\nabla_{\theta}L_{\lambda^l}(\theta^{l,k})\right)^T
\nabla_{\theta}L_{\lambda^l}(\theta^{l,k})\\
&\geq\|\nabla_{\theta}L_{\lambda^l}(\theta^{l,k})\|_2^2-\|\nabla_{\theta}
L_{\lambda^l}(\theta^{l,k})\|_2(k+1)^{-\frac{1+\eta_0}{2+\eta_0}}\\
&\geq \|\nabla_{\theta}L_{\lambda^l}(\theta^{l,k})\|_2^2-\left(\dfrac{2}
{(1-\gamma)^2}+2\bar{\lambda}\right)(k+1)^{-\frac{1+\eta_0}{2+\eta_0}},
\end{split}
\]
where the last two steps used Cauchy inequality, \eqref{Eklambda_error}
and the fact that by \eqref{pg_lambda},
\[
\begin{split}
\|\nabla_{\theta}&L_{\lambda^l}(\theta^{l,k})\|_2\\
&\leq \sum_{t=0}^{\infty}\gamma^t\Expect_{l,k}\left(Q^{\pi_{\theta}^{l,k}}
(s_t^{l,k},a_t^{l,k})\left\|\nabla_{\theta}
\log\pi_{\theta^{l,k}}(a_t^{l,k}|s_t^{l,k})\right\|_2\right)
+\dfrac{\bar{\lambda}}{SA}\sum_{s\in\mathcal{S},a\in\mathcal{A}}
\|\nabla_{\theta}\log\pi_{\theta^{l,k}}(a|s)\|_2\\
&\leq 2/(1-\gamma)^2+2\bar{\lambda}.
\end{split}
\]
Hence we can take $C_2=1$ and
$\delta_{l,k}=\left(\frac{2}{(1-\gamma)^2}+2\bar{\lambda}\right)
(k+1)^{-\frac{1+\eta_0}{2+\eta_0}}$. Thus we have 
\begin{align*}
\sum_{k=0}^{T_l-1}\delta_{l,k}^2 &= \left(\frac{2}{(1-\gamma)^2}+
2\bar{\lambda}\right)^2\sum_{k=0}^\infty(k+1)^{-\frac{2+2\eta_0}
{2+\eta_0}}\\
&\leq 8\left(\frac{1}{(1-\gamma)^2}+\bar{\lambda}\right)^2
\left(1+\dfrac{1}{\eta_0}\right),
\end{align*}
and hence we can take
$C=8\left(\frac{1}{(1-\gamma)^2}+\bar{\lambda}\right)^2
\left(1+\frac{1}{\eta_0}\right)$. Notice that for notational simplicity,
we have taken $\eta_0=1$ in the statement of the proposition.
\vspace{1em}

\noindent\textbf{Validation of bounded second-order moment growth.}
Finally, we bound the second-order moment of the policy gradient.
In the following, for a random vector $X=(X_1,\dots,X_n)\in\reals^n$, we define
$\V X=\sum_{i=1}^n\var X_i$, and similarly $\V_{l,k} X=
\sum_{i=1}^n\var_{l,k} X_i$, where $\var_{l,k}$ denotes the conditional
variance given the $(l,k)$-th iteration $\theta^{l,k}$. Now define the
constant $\bar{V}$ as the uniform upper bound on the variance of
the policy gradient vector, \ie,
\[
\bar{V}=\sup_{H\geq0,\, \theta\in\Theta,\,\lambda\in[0,\bar{\lambda}]}
\V\left(\sum_{t=0}^{\lfloor \beta H\rfloor}\gamma^t\widehat{Q}(s_t,a_t)
\nabla_{\theta}\log\pi_{\theta}(a_t|s_t)+\dfrac{\lambda}{SA}
\sum_{s\in\mathcal{S},a\in\mathcal{A}}\nabla_{\theta}\log\pi_{\theta}(a|s)\right),
\]
where $\tau=(s_0,a_0,r_0,\dots,s_H,a_H,r_H)$ is sampled from
$\mathcal{M}$ following policy $\pi_{\theta}$, and
$\hat{Q}(s_t,a_t)=\sum_{t'=t}^H\gamma^{t'-t}r_{t'}$.

Then we have $\V_{l,k}\widehat{\nabla}_{\theta}
L_{\lambda^l}(\theta^{l,k})\leq \bar{V}$ for any $l,\,k\geq 0$ by definition.
In addition, since for any random vector $X\in\reals^n$,
\[
\Var X\leq
\sum_{i=1}^n\Expect X_i^2=\Expect\|X\|_2^2,
\] we have by the same
argument as \eqref{Lhat_asbound} that
\begin{align*}
\bar{V}&\leq \Expect\left\|\sum_{t=0}^{\lfloor \beta H\rfloor}
\gamma^t\widehat{Q}(s_t,a_t)\nabla_{\theta}\log\pi_{\theta}(a_t|s_t)
+\dfrac{\lambda}{SA}\sum_{s\in\mathcal{S},a\in\mathcal{A}}
\nabla_{\theta}\log\pi_{\theta}(a|s)\right\|_2^2 \\
&\leq
\dfrac{4(1+\bar{\lambda}(1-\gamma)^2)^2}{(1-\gamma)^4}.
\end{align*}

Finally, since for any random vector $X\in\reals^n$,
\[
\Expect_{l,k}\|X\|^2=\Expect_{l,k}\sum\nolimits_{i=1}^nX_i^2=
\sum\nolimits_{i=1}^n(\Expect_{l,k}X_i^2+\var\nolimits_{l,k}X_i)=
\|\Expect_{l,k}X\|_2^2+\V\nolimits_{l,k}X,
\] we have
\[
\begin{split}
\Expect_{l,k}\left\|\widehat{\nabla}_{\theta}
L_{\lambda^l}(\theta^{l,k})\right\|_2^2&\leq \|J_1+J_2+J_3\|_2^2+\bar{V}\\
&\leq 2\|J_1\|_2^2+2(\|J_2\|_2+\|J_3\|_2)^2+\bar{V}\\
&\leq 2\|\nabla_{\theta}L_{\lambda^l}(\theta^{l,k})\|_2^2+
\dfrac{32\gamma^{2\min\{\beta,1-\beta\}H^k}}{(1-\gamma)^4}+\bar{V}\\
&\leq
2\|\nabla_{\theta}L_{\lambda^l}(\theta^{l,k})\|_2^2+
\dfrac{32}{(1-\gamma)^4}+\bar{V},
\end{split}
\]
and hence we can take $M_2=2$ and $M_1=32/(1-\gamma)^4+\bar{V}$.
This completes our proof.
\end{proof}

\subsection{Proof of Lemma \ref{verify_basic_reinforce_base}}
\begin{proof}
The proof is nearly identical to the case when $b= 0$ above.
Hence we only outline the proof while highlighting the differences.

Firstly, similar to \eqref{Lhat_asbound}, we have
\[\left\|\widehat{\nabla}_{\theta} L_{\lambda^{l}}(\theta^{l,k})\right\|_2\leq
\left(\frac{1}{1-\gamma}+B\right)\frac{2}{1-\gamma}+2\bar{\lambda} \quad \text{a.s.,}
\]
and hence we can take $C_1=\frac{2+2B(1-\gamma)}{(1-\gamma)^2}+2\bar{\lambda}$.

Secondly, by the proof of \cite[Lemma 4.10]{agarwal2019reinforcement},
we have
\BEQ
\begin{split}
\Expect_{l,k}\left(\sum_{t=0}^{\lfloor \beta H^{l,k}\rfloor}\gamma^tb(s_t^{l,k})
\nabla_{\theta}\log\pi_{\theta^{l,k}}(a_t^{l,k}|s_t^{l,k})\right)=0.
\end{split}
\EEQ
Hence  $\Expect_{l,k}\widehat{\nabla}_{\theta}L_{\lambda^l}(\theta^{l,k})$
is the same as in the proof when $b= 0$ above,
and hence we can take
\[
\begin{split}
&\text{$C_2=1$,\quad
$C=16\left(\frac{1}{(1-\gamma)^2}+\bar{\lambda}\right)^2$,}\\
&\text{$\delta_{l,k}=\left(\frac{2}{(1-\gamma)^2}+2\bar{\lambda}\right)(k+1)^{-\frac{2}{3}}$,
\quad $H^{l,k}\geq \frac{3\log_{1/\gamma}\left(\frac{8(k+1)}{(1-\gamma)^{3}}\right)}
{2\min\{\beta,1-\beta\}}$.}
\end{split}
\]

~\\

Finally, by definition, $\bar{V}_b$ can be written explicitly as
\[
\bar{V}_b=\sup_{H\geq0,\, \theta\in\Theta,\,\lambda\in[0,\bar{\lambda}]}
\V\left(\sum_{t=0}^{\lfloor \beta H\rfloor}\gamma^t(\widehat{Q}(s_t,a_t)-b(s_t))
\nabla_{\theta}\log\pi_{\theta}(a_t|s_t)+\dfrac{\lambda}{SA}
\sum_{s\in\mathcal{S},a\in\mathcal{A}}\nabla_{\theta}\log\pi_{\theta}(a|s)\right),
\]
where $\tau=(s_0,a_0,r_0,\dots,s_H,a_H,r_H)$ is sampled from
$\mathcal{M}$ following policy $\pi_{\theta}$, and
$\hat{Q}(s_t,a_t)=\sum_{t'=t}^H\gamma^{t'-t}r_{t'}$.

Hence similar to $\bar{V}$ in then proof when $b=0$ above,
we have $\bar{V}_b\leq 4\left(\frac{1+B(1-\gamma)}{(1-\gamma)^2}+
\bar{\lambda}\right)^2$, $M_2=2$ and
$M_1=32/(1-\gamma)^4+\bar{V}_b$.
\end{proof}

\section{Proofs for convergence analysis}
Firstly, we state a simple result from elementary analysis, which will be used
repeatedly in our proof below. 
\begin{lemma}\label{simple_alpha_bound}
Let $x_k=\frac{1}{\sqrt{k+3}\log_2(k+3)}$ for $k\geq 0$. Then we have
\[
\begin{split}
&\sum_{k=K_1}^{K_2}x_k\geq \dfrac{K_2-K_1+1}{\sqrt{K_2+3}\log_2(K_2+3)},\quad 
\sum_{k=0}^{\infty}x_k^4\leq\sum_{k=0}^{\infty}x_k^2\leq 1. 
\end{split}
\]
\end{lemma}
\begin{proof}
The first inequality immediately comes from the fact that $x_k$ is monotonically 
decreasing in $k$. The second inequality can be derived by noticing that $x_k<1$
for any $k\geq 0$, and that
\[
\sum_{k=0}^{\infty}x_k^2\leq \int_0^{\infty}\dfrac{1}{(x+2)(\log_2(x+2))^2}dx=1.
\]
This completes the proof.
\end{proof}

\subsection{Proof of Lemma \ref{highprob2as}}
\begin{proof}
The proof is a direct application of the well-known Borel-Cantelli lemma
 \cite{klenke2013probability}. Let $\delta_N=1/N^2$ and define the events 
 $\{\bar{A}_N\}_{N\geq 0}$ as
 \[
\bar{A}_N=\{\mathbf{regret}(N)>d_1(N+c)^{d_2}(\log(N/\delta_N))^{d_3}+d_4(\log N)^{d_5}\}.
 \]
 Then $\prob(\bar{A}_N)\leq \delta_N$, and hence 
 $\sum_{N=1}^{\infty}\prob(\bar{A}_N)\leq \sum_{N=1}^{\infty}1/N^2<\infty$. Hence 
 by Borel-Cantelli lemma, we have
 \[
 \prob(\bar{A}_N\text{ occurs infinitely often})=0.
 \]
Finally, by noticing that the complement of 
 $\bar{A}_N$ is a subset of $A_N$, the proof is complete.  
\end{proof}

\subsection{Proof of Theorem \ref{phase_regret}}
\begin{theorem}[Formal statement of Theorem \ref{phase_regret}]\label{phase_regret_formal}
Under Assumptions \ref{softmax-reg} and \ref{unbiased-boundedvar}, for phase
$l\geq 0$ suppose that we choose
$\alpha^{l,k}=C_{l,\alpha}\frac{1}{\sqrt{k+3}\log_2(k+3)}$ for some
 $C_{l,\alpha}\in(0,C_2/(M_2\beta_{\lambda^l})]$.
Then for any $\epsilon>0$, if we choose
$\lambda^l=\frac{\epsilon(1-\gamma)}{2}$,
then 
for any $\delta\in(0,1)$, with probability at least $1-\delta$,
for any $K\in\{0,\dots,T_l-1\}$, we have
\BEQ\label{regret_phase_l_formal}
\begin{split}
\mathbf{regret}_l(K)\leq&\, \dfrac{4(D_{l}+\sqrt{2C_{l}
\log(2/\delta)})}{(1-\gamma)C_2E_l\epsilon^2}\sqrt{K+1}\log_2(K+3)
+\epsilon (K+1)\left\|\dfrac{d_\rho^{\pi^\star}}{\rho}\right\|_\infty\\
&+\frac{\gamma+\gamma\log(K+1)}{1-\gamma}.
\end{split}
\EEQ
Here $E_l=\frac{C_{l,\alpha}(1-\gamma)^2}{16S^2A^2}$, and
\[
\begin{split}
C_{l}&=32C_1^2C_{l,\alpha}^2\left(\dfrac{1}{(1-\gamma)^2}
+\lambda^l\right)^2+
\dfrac{\beta_{\lambda^l}^2C_1^4C_{l,\alpha}^4}{2},\\
D_{l}&=CC_{l,\alpha}^2
+\beta_{\lambda^l}M_1C_{l,\alpha}^2
+F^\star-L_{\lambda^l}(\theta^{l,0}).
\end{split}
\]
\end{theorem}
\begin{proof}
The proof consists of two parts. In the first part, we establish an upper
bound on the weighted gradient norms sum of previous iterates in the
current phase. The second part then utilizes this bound to establish an
upper bound on the phase regret.
\item
\paragraph{Bounding the weighted gradient norms sum.}
By Proposition \ref{Lsmooth} and an equivalent definition of strongly
smoothness (\cf \cite[Appendix]{MonoPrimer}), we have
\[
\begin{split}
-L_{\lambda^l}(\theta^{l,k+1})-(-L_{\lambda^l}(\theta^{l,k}))&\leq
-\nabla_{\theta}L_{\lambda^l}(\theta^{l,k})^T(\theta^{l,k+1} -
\theta^{l,k})+\dfrac{\beta_{\lambda^l}}{2}\|\theta^{l,k+1}-\theta^{l,k}\|_2^2\\
&=\underbrace{-\alpha^{l,k}\nabla_{\theta}L_{\lambda^l}(\theta^{l,k})^T
\widehat{\nabla}_{\theta} L_{\lambda^l}(\theta^{l,k})+\dfrac{\beta_{\lambda^l}
(\alpha^{l,k})^2}{2}\|\widehat{\nabla}_{\theta}
L_{\lambda^l}(\theta^{l,k})\|_2^2}_{Y_{l,k}}.
\end{split}
\]

Let $Z_{l,k}= Y_{l,k}-\Expect_{l,k}[Y_{l,k}]$.
Then the above inequality implies that
\BEQ\label{k-step_ineq}
\begin{split}
L_{\lambda^l}&(\theta^{l,k}) - L_{\lambda^l}(\theta^{l,k+1})\\
\leq& -\alpha^{l,k}
L_{\lambda^l}(\theta^{l,k})^T\Expect_{l,k}\widehat{\nabla}_{\theta}
L_{\lambda^l}(\theta^{l,k})+\dfrac{\beta_{\lambda^l}
(\alpha^{l,k})^2}{2}\Expect_{l,k}\|\widehat{\nabla}_{\theta}
L_{\lambda^l}(\theta^{l,k})\|_2^2+Z_{l,k}\\
\leq& -\alpha^{l,k}\left(C_2\|\nabla_{\theta}
L_{\lambda^l}(\theta^{l,k})\|_2^2-\delta_{l,k}\right)+
\dfrac{\beta_{\lambda^l}(\alpha^{l,k})^2}{2}\left(M_1+
M_2\|\nabla_{\theta}L_{\lambda^l}(\theta^{l,k})\|_2^2\right)+Z_{l,k}\\
=&-\alpha^{l,k}(C_2-M_2\beta_{\lambda^l}\alpha^{l,k}/2)\|\nabla_{\theta}
L_{\lambda^l}(\theta^{l,k})\|_2^2+\alpha^{l,k}\delta_{l,k}+
\dfrac{\beta_{\lambda^l}M_1(\alpha^{l,k})^2}{2}+Z_{l,k}\\
\leq& -\dfrac{C_2\alpha^{l,k}}{2}\|\nabla_{\theta}L_{\lambda^l}
(\theta^{l,k})\|_2^2+\alpha^{l,k}\delta_{l,k}+
\dfrac{\beta_{\lambda^l}M_1(\alpha^{l,k})^2}{2}+Z_{l,k}.
\end{split}
\EEQ

Now define $X_{l,K}=\sum_{k=0}^{K-1} Z_{l,k}$ (with $X_{l,0}=0$), then
\BEQ\label{cond_expect_mart}
\Expect(X_{l,K+1}|\mathcal{F}_{l,K})=\sum_{k=0}^{K-1}Z_{l,k}+
\Expect(Y_{l,K}-\Expect_{l,K}Y_{l,K}|\mathcal{F}_{l,K})=X_{l,K}.
\EEQ
Here $\mathcal{F}_{l,K}$ is the filtration up to episode $K$ in phase $l$,
\ie, the $\sigma$-algebra generated by all iterations 
$\{\theta^{0,0},\dots,\theta^{0,T_0},\dots,\theta^{l,0},\dots,\theta^{l,K}\}$ up to
the $(l,K)$-th one.
Notice that the second equality makes use of the fact that 
given the current policy,
the correspondingly 
sampled trajectory is conditionally independent of all previous policies and trajectories.

In addition, for any $K\geq 1$,
\[
\begin{split}
|X_{l,K}-X_{l,K-1}|=&\,|Z_{l,K-1}|\leq \alpha^{l,K-1}\|\nabla_{\theta}
L_{\lambda^{l}}(\theta^{l,K-1})\|_2\|\Expect_{l,K-1}
\widehat{\nabla}_{\theta}L_{\lambda^{l}}(\theta^{l,K-1})-
\widehat{\nabla}_{\theta}L_{\lambda^{l}}(\theta^{l,K-1})\|_2\\
&+\dfrac{\beta_{\lambda^l}(\alpha^{l,K-1})^2}{2}\left|\Expect_{l,K-1}
\|\widehat{\nabla}_{\theta} L_{\lambda^l}(\theta^{l,K-1})\|_2^2-
\|\widehat{\nabla}_{\theta} L_{\lambda^l}(\theta^{l,K-1})\|_2^2\right|\\
\leq&\,\underbrace{2C_1\left(\dfrac{2}{(1-\gamma)^2}+2\lambda^l\right)
\alpha^{l,K-1}+\dfrac{\beta_{\lambda^l}}{2}
C_1^2(\alpha^{l,K-1})^2}_{c_{l,K}}.
\end{split}
\]
Here we use the fact that
\[
\|\nabla_{\theta}L_{\lambda^{l}}(\theta^{l,K-1})\|_2
\leq 2/(1-\gamma)^2+2\lambda^l,
\] which follows from the same argument
as \eqref{Lhat_asbound}. The above inequality on $|X_{l,K}-X_{l,K-1}|$ 
also implies that
$\Expect|X_{l,K}|<\infty$, which, together with \eqref{cond_expect_mart},
implies that $X_{l,K}$ is a martingale. 
Notice that although $X_{l,K}$ is only defined for $K=0,\dots,T_l$, we can 
augment the sequence by setting $X_{l,K}=X_{l,T_l}$ and 
$\mathcal{F}_{l,K}=\mathcal{F}_{l,T_l}$ for all $K>T_l$, and it's obvious that
\eqref{cond_expect_mart} and $\Expect|X_{l,K}|<\infty$ also hold for $K\geq T_l$. 
And by saying that 
$X_{l,K}$ is a martingale, we are indeed referring to this (infinite length) 
 augmented sequence of random variables. 

Now by the definition of $\alpha^{l,k}$, it's easy to see that
$\sum_{K=1}^{T_l}c_{l,K}^2\leq C_{l}<\infty$, where
\BEQ\label{C_l_ub}
C_{l}=32C_1^2C_{l,\alpha}^2\left(\dfrac{1}{(1-\gamma)^2}
+\lambda^l\right)^2+
\dfrac{\beta_{\lambda^l}^2C_1^4C_{l,\alpha}^4}{2}.
\EEQ
Hence by 
Azuma-Hoeffding inequality,   
for any $c>0$,  
\BEQ\label{X_l_inf_bd}
\prob(|X_{l,T_l}|\geq c)\leq 2e^{-c^2/(2C_{l})}.
\EEQ

Then by summing up the inequalities \eqref{k-step_ineq} from
$k=0$ to $K$, we obtain that
\BEQ
\begin{split}
\dfrac{C_2}{2}&\sum_{k=0}^{K}\alpha^{l,k}\|\nabla_{\theta}
L_{\lambda^l}(\theta^{l,k})\|_2^2\leq \dfrac{C_2}{2}
\sum_{k=0}^{T_l-1}\alpha^{l,k}\|\nabla_{\theta}
L_{\lambda^l}(\theta^{l,k})\|_2^2\\
&\leq \sum_{k=0}^\infty\alpha^{l,k}\delta_{l,k}+\dfrac{\beta_{\lambda^l}
M_1\sum_{k=0}^\infty(\alpha^{l,k})^2}{2}+\sum_{k=0}^{T_l-1}Z_{l,k}+
\sup_{\theta\in\Theta}L_{\lambda^l}(\theta)-L_{\lambda^l}(\theta^{l,0})\\
&\leq \sum_{k=0}^{\infty}(\alpha^{l,k})^2\sum_{k=0}^{\infty}\delta_{l,k}^2+
\dfrac{\beta_{\lambda^l}M_1}{2}\sum_{k=0}^{\infty}(\alpha^{l,k})^2+
X_{l,T_l}+F^\star-L_{\lambda^l}(\theta^{l,0})\\
&\leq \underbrace{CC_{l,\alpha}^2
+\beta_{\lambda^l}M_1C_{l,\alpha}^2
+F^\star-L_{\lambda^l}(\theta^{l,0})}_{D_{l}}+X_{l,T_l},
\end{split}
\EEQ
where we use the fact that the regularization term $R(\theta)\leq 0$
for all $\theta\in\Theta$.

Hence we have
\BEQ\label{weighted_grad_norm_ub}
\sum_{k=0}^{K}\alpha^{l,k}\|\nabla_{\theta}L_{\lambda^l}(\theta^{l,k})\|_2^2
\leq \frac{2(D_{l}+X_{l,T_l})}{C_2}.
\EEQ

\item\paragraph{Bounding the phase regret.} We now establish the regret
bound in phase $l$ using \eqref{weighted_grad_norm_ub}.

Fix $l\geq 0$ and $K\in\{0,\dots,T_l-1\}$. Let
\[
I^+ = \{k\in\{0,\dots,K\}\,|\,\|\nabla_{\theta}L_{\lambda^l}(\theta^{l,k})\|_2
\geq \lambda^l/(2SA)\}.
\]
For simplicity, assume for now that $|I^+|>0$.
Then since $\alpha^{l,k}$ is decreasing in $k$, we have
\[
\begin{split}
2(D_{l}+X_{l,T_l})/C_2&\geq \dfrac{(\lambda^l)^2}{4S^2A^2}
\sum_{k=K-|I^+|+1}^{K}\alpha^{l,k}\\
&=\epsilon^2 \underbrace{\frac{C_{l,\alpha}(1-\gamma)^2}{16S^2A^2}}_{E_l}\sum_{k=K-|I^+|+1}^K\frac{1}{\sqrt{k+3}\log_2(k+3)}\\
(\text{By Lemma \ref{simple_alpha_bound}})&\geq E_l\epsilon^2\dfrac{|I^+|}{\sqrt{K+3}\log_2(K+3)}.
\end{split}
\]
Hence we have (by the simple fact that $\sqrt{K+3}\leq 2\sqrt{K+1}$ for any $K\geq 0$)
\BEQ\label{prop_bd}
|I^+|\leq \dfrac{4(D_{l}+X_{l,T_l})}{C_2
E_l\epsilon^2}\sqrt{K+1}\log_2(K+3)
\EEQ

Now by Proposition  \ref{kakade_prop} and the choice of $\lambda^l$,
we have that for any $k\notin I^+$,
\[
F^\star-F(\pi_{\theta^{l,k}})\leq \|d_\rho^{\pi^\star}/\rho\|_\infty\epsilon.
\]
Since for any $\pi\in\Pi$, $F(\pi)\in[0,1/(1-\gamma)]$, we have
$F^\star-F(\pi)\leq 1/(1-\gamma)$. Hence by \eqref{prop_bd}, we have
\[
\begin{split}
\sum\nolimits_{k\leq K}F^\star&-F(\pi_{\theta^{l,k}})\\
&=\sum_{k\in I^+}F^\star-F(\pi_{\theta^{l,k}})+\sum_{k\notin I^+}F^\star-
F(\pi_{\theta^{l,k}})\\
&\leq |I^+|/(1-\gamma)+\left(K+1-|I^+|\right)\|d_\rho^{\pi^\star}/\rho\|_\infty\epsilon\\
&\leq \dfrac{4(D_{l}+X_{l,T_l})}{(1-\gamma)C_2E_l
\epsilon^2}\sqrt{K+1}\log_2(K+3)+(K+1)\|d_\rho^{\pi^\star}/\rho\|_\infty\epsilon.
\end{split}
\]

This immediately implies that
\BEQ\label{regret_l_K_tmpbd}
\begin{split}
&\mathbf{regret}_l(K)=\sum\nolimits_{k\leq K}F^\star-F(\pi_{\theta^{l,k}})
+\sum\nolimits_{k\leq K}F(\pi_{\theta^{l,k}})-\hat{F}^{l,k}(\pi_{\theta^{l,k}})\\
&\leq \sum\nolimits_{k\leq K}F^\star-F(\pi_{\theta^{l,k}})+
\sum\nolimits_{k\leq K}\Expect_{l,k}
\sum\nolimits_{t=H^{l,k}+1}^{\infty}\gamma^tr(s_t^{l,k},a_t^{l,k})\\
&\leq  \sum\nolimits_{k\leq K}F^\star-F(\pi_{\theta^{l,k}})
+\sum_{k\leq K}\frac{\gamma/(k+1)}{1-\gamma}\\
&\leq  \dfrac{4(D_{l}+X_{l,T_l})}{(1-\gamma)C_2E_l
\epsilon^2}\sqrt{K+1}\log_2(K+3)+(K+1)\left\|\dfrac{d_\rho^{\pi^\star}}{\rho}\right\|_\infty\epsilon
+\frac{\gamma+\gamma\log(K+1)}{1-\gamma}.
\end{split}
\EEQ

Now if $|I^+|=0$, then we immediately have that
\[
\mathbf{regret}_l(K)\leq (K+1)\left\|\dfrac{d_\rho^{\pi^\star}}{\rho}\right\|_\infty\epsilon+\frac{\gamma+\gamma\log(K+1)}{1-\gamma},
\]
and hence \eqref{regret_l_K_tmpbd} always holds.

Finally, by \eqref{X_l_inf_bd}, we have that with probability at least
$1-\delta$, for all $K\in\{0,\dots,T_l-1\}$,
\[
\begin{split}
&\mathbf{regret}_l(K)\\
&\leq \dfrac{4(D_{l}+
\sqrt{2C_{l}\log(2/\delta)})}{(1-\gamma)C_2E_l\epsilon^2}
\sqrt{K+1}\log_2(K+3)+\epsilon (K+1)\left\|\dfrac{d_\rho^{\pi^\star}}{\rho}\right\|_\infty+\frac{\gamma+\gamma\log(K+1)}{1-\gamma}.
\end{split}
\]
This completes our proof.
\end{proof}

\subsection{Overall regret bound for general policy gradient estimators}
\label{overall_regret_general}
In this section, we state and prove the overall regret bound for general policy gradient estimators,
which generalizes Theorem \ref{regret_reinforce} for REINFORCE gradient estimators. 
\begin{theorem}[General regret bound]
\label{overall_regret}
Under Assumptions \ref{softmax-reg} and \ref{unbiased-boundedvar}, suppose
that for each $l\geq 0$, we choose
$\alpha^{l,k}=C_{l,\alpha}\frac{1}{\sqrt{k+3}\log_2(k+3)}$ for some
 $C_{l,\alpha}\in[\underline{C}^{\alpha},C_2/(M_2\beta_{\lambda^l})]$,
with $\underline{C}^\alpha\in(0,C_2/(M_2\beta_{\bar{\lambda}})]$ and
$\bar{\lambda}=\frac{1-\gamma}{2}$. In addition,
suppose that we specify $T_0\geq 1$, choose  $T_l=2^lT_0$,
$\epsilon^l=T_l^{-1/6}$ and
$\lambda^l=\frac{\epsilon^l(1-\gamma)}{2}$
for each $l\geq 0$. Then we have
 for any $\delta\in(0,1)$,
with probability at least $1-\delta$, for any $N\geq 0$, we have
\BEQ\label{regret_all_formal}
\begin{split}
\mathbf{regret}(N)\leq \bar{R}_1(N)+\bar{R}_2(N)
=O(N^{5/6}(\log(N/\delta))^{5/2}),
\end{split}
\EEQ
where 
\BEQ\label{regret_all_formal_details}
\begin{split}
\bar{R}_1(N)=&\, \left(\dfrac{4(\bar{D}+\sqrt{2\bar{C}
((\log_2 (N+1)+2)\log2+\log(1/\delta))})}{(1-\gamma)C_2\underline{E}}+\left\|\dfrac{d_\rho^{\pi^\star}}{\rho}\right\|_\infty\right)\\
&\,\times(N+T_0)^{\frac{5}{6}}(\log_2(2N+2T_0+1))^2,\\
\bar{R}_2(N)=&\, \frac{(\log_2 (N+1)+1)(\gamma+\gamma\log (N+T_0))}{1-\gamma}.
\end{split}
\EEQ
Here
the constants  $\underline{E}=\frac{\underline{C}^{\alpha}(1-\gamma)^2}
{16S^2A^2}$, 
\[
\bar{D}=C\bar{C}_{\alpha}^2+
\beta_{\bar{\lambda}}M_1\bar{C}_{\alpha}^2
+\frac{1}{1-\gamma}+\log(1/\epsilon_{\rm pp}),\]
\[
\bar{C}=32C_1^2\bar{C}_{\alpha}^2\left(\frac{1}{(1-\gamma)^2}
+\bar{\lambda}\right)^2+
\frac{\beta_{\bar{\lambda}}^2C_1^4\bar{C}_{\alpha}^4}{2},
\]
with $\bar{C}_{\alpha}=\frac{C_2(1-\gamma)^3}{8M_2}$.

In addition, we also have
\[
\lim_{N\rightarrow\infty}\mathbf{regret}(N)/(N+1)=0\quad \text{almost surely}
\]
with an asymptotic rate of $O(N^{-1/6}(\log N)^{5/2})$. 
\end{theorem}
\begin{remark}
Notice that the
constant $\underline{E}$ is a uniform lower bound of $E_l$ ($l\geq 0$), while
the constants $\bar{D}$ and $\bar{C}$ are
uniform upper bounds of $D_{l}$ and $C_{l}$ ($l\geq 0$), respectively. Here 
the constants $E_l,\,D_l,\,C_l$ are those defined in Theorem \ref{phase_regret_formal}.
\end{remark}
\begin{remark}
In the 
big-$O$ notation above, we have (temporarily) hidden the problem dependent 
quantities, which will be made explicit when we specialize the results to
the REINFORCE gradient estimation below. 
\end{remark}
\begin{proof}
We first prove the high probability result. 
By \eqref{regret_phase_l_formal} and the choices of $\epsilon^l$ and $\lambda^l$,
we have that for any phase $l\geq 0$, with probability at least
$1-\delta/2^{l+1}$, for all $K\in\{0,\dots,T_l-1\}$,
\[
\begin{split}
&\mathbf{regret}_l(K)\\
&\leq \left(\dfrac{4(\bar{D}+
\sqrt{2\bar{C}((l+2)\log2+\log(1/\delta))})}{(1-\gamma)
C_2\underline{E}}+\left\|\dfrac{d_\rho^{\pi^\star}}{\rho}\right\|_\infty\right)T_l^{5/6}\log_2(T_l+2)+
\frac{\gamma+\gamma\log T_l}{1-\gamma}.
\end{split}
\]
where $\underline{E}=\frac{\underline{C}^{\alpha}(1-\gamma)^2}
{16S^2A^2}$,
\[
\bar{D}=C\bar{C}_{\alpha}^2+
\beta_{\bar{\lambda}}M_1\bar{C}_{\alpha}^2
+\frac{1}{1-\gamma}+\log(1/\epsilon_{\rm pp}),\]
\[
\bar{C}=32C_1^2\bar{C}_{\alpha}^2\left(\frac{1}{(1-\gamma)^2}
+\bar{\lambda}\right)^2+
\frac{\beta_{\bar{\lambda}}^2C_1^4\bar{C}_{\alpha}^4}{2},
\]
with $\bar{C}_{\alpha}=\frac{C_2(1-\gamma)^3}{8M_2}$ and
$\bar{\lambda}=\frac{1-\gamma}{2}$.
Here we used the fact that $\epsilon^l\leq 1$, which then implies that
$\lambda^l\leq \bar{\lambda}$ and
\[
\frac{8}{(1-\gamma)^3}\leq
\beta_{\lambda^l}\leq \beta_{\bar{\lambda}}=\frac{8}{(1-\gamma)^3}+
\frac{2\bar{\lambda}}{S}.
\]
We also used the fact that $F^\star-F(\pi)
\leq 1/(1-\gamma)$ for any $\pi\in\Pi$ and that by the definition of
\texttt{PostProcess}, $R_{\lambda^l}(\pi_{\theta^{l,0}})\geq \log
\epsilon_{\rm pp}$.

Now recall that for any $N\geq0$, we have 
\[
\begin{split}
\mathbf{regret}(N) &= \sum_{l=0}^{l_N-1}\mathbf{regret}_l(T_l-1)+
\mathbf{regret}_{l_N}(k_N) \\
&\leq \sum_{l=0}^{l_N}\mathbf{regret}_l(T_l-1),
\end{split}
\]
where $(l_N,k_N)=G_{\mathcal{T}}(N)$. In addition, by the choices of $T_l$, we have
that for any $0\leq k\leq T_l-1$,
\begin{align*}
B_{\mathcal{T}}(l,k)&=\sum_{j=0}^{l-1}T_j+k\\
&=(2^l-1)T_0+k\\
&\geq (2^l-1)T_0.
\end{align*} Hence for any $N\geq 0$, we have $l_N\leq \log_2
\left(\frac{N}{T_0}+1\right)\leq \log_2(N+1)$.

Thus we have that  
with probability at least
$1-\sum_{l=0}^{l_N}\delta/2^{l+1}\geq 1-\delta$, for any $N\geq 0$,
\[
\begin{split}
&\mathbf{regret}(N)\leq (l_N+1)(\hat{R}_1(N)+\hat{R}_2(N)),
\end{split}
\]
where 
\[
\begin{split}
\hat{R}_1(N)&=\left(\dfrac{4(\bar{D}+
\sqrt{2\bar{C}((l_{N}+2)\log2+\log(1/\delta))})}
{(1-\gamma)C_2\underline{E}}+\left\|\dfrac{d_\rho^{\pi^\star}}{\rho}\right\|_\infty\right)(N+T_0)^{\frac{5}{6}}\log_2(N+T_0+2),\\
\hat{R}_2(N)&=\frac{\gamma+\gamma\log (N+T_0)}{1-\gamma}.
\end{split}
\]
Finally, noticing that $l_N+1\leq\log_2(N+1)+1\leq \log_2(2N+2T_0+1)$, we have
\[
\begin{split}
(l_N+1)\hat{R}_1(N)\leq&\, \left(\dfrac{4(\bar{D}+\sqrt{2\bar{C}
((\log_2 (N+1)+2)\log2+\log(1/\delta))})}{(1-\gamma)C_2\underline{E}}+\left\|\dfrac{d_\rho^{\pi^\star}}{\rho}\right\|_\infty\right)\\
&\,\times(N+T_0)^{\frac{5}{6}}(\log_2(2N+2T_0+1))^2,\\
(l_N+1)\hat{R}_2(N)\leq&\, \frac{(\log_2 (N+1)+1)(\gamma+\gamma\log (N+T_0))}{1-\gamma},
\end{split}
\]
which immediately imply \eqref{regret_all_formal} and \eqref{regret_all_formal_details}. 
Notice that here we used the fact that $\log_2(N+1)+1\leq \log_2(2N+2T_0+1)$ (since
$T_0\geq 1$), and that
 $T_l\leq N+1\leq N+T_0$ for all
$l=0,\dots,l_{N}-1$ and $T_{l_N}=2^{l_N}T_0\leq N+T_0$. 

Finally, by invoking Lemma \ref{highprob2as}, we immediately 
obtain the almost sure convergence result. This completes our proof.
\end{proof}

\section{Formal statement of REINFORCE regret bound}
\label{formal_statement}
Here we provide a slightly more formal restatement of Theorem \ref{regret_reinforce}, 
with details about the constants in the big-$O$ notation in the main text. 
Recall that our goal there is specialize
(and slightly simplify) the regret bound in Theorem \ref{overall_regret} to the
case when the
REINFORCE gradient estimation in \S\ref{reinforce_grad_est} is adopted
to evaluate $\widehat{\nabla}_{\theta} L_{\lambda^l}(\theta^{l,k})$. In
particular, 
we have the following corollary. 
The proof is done by simply combining Lemma
\ref{verify_basic_reinforce_base} (with $\lambda^l\leq \bar{\lambda}$ by their 
definitions in Theorem \ref{regret_reinforce} or Corollary 
\ref{regret_reinforce_formal} below) and Theorem \ref{overall_regret},
together with the
specific choices of the hyper-parameters as well as the constants in Lemma
\ref{verify_basic_reinforce_base}
plugged in and some elementary algebraic
simplifications,
and is hence omitted.
\begin{corollary}[Formal statement of Theorem \ref{regret_reinforce}]
\label{regret_reinforce_formal}
Under Assumption \ref{softmax-reg},
suppose that for each $l\geq 0$, we choose
$\alpha^{l,k}=C_{l,\alpha}\frac{1}{\sqrt{k+3}\log_2(k+3)}$, with
$C_{l,\alpha}\in[\underline{C}^\alpha, 1/(2\beta_{\lambda^l})]$,
$\underline{C}^\alpha\in(0,1/(2\beta_{\bar{\lambda}})]$ and
$\bar{\lambda}=\frac{1-\gamma}{2}$, and choose
$T_l=2^l$, $\epsilon^l=T_l^{-1/6}=2^{-l/6}$,
$\lambda^l=\frac{\epsilon^l(1-\gamma)}
{2}$ and $\epsilon_{\rm pp}=1/(2A)$.
In addition, suppose that \eqref{reinforce_grad_base} is adopted to
evaluate $\widehat{\nabla}_{\theta} L_{\lambda^l}(\theta^{l,k})$,
with $\beta\in(0,1)$, $|b(s)|\leq B$ for any $s\in\mathcal{S}$ (where $B>0$ is a
constant), and
that \eqref{Hlk_bound} holds for $H^{l,k}$ for all $l,\,k\geq 0$. Then we have 
for any $\delta\in(0,1)$, with probability at least $1-\delta$, for any $N\geq 0$,
we have
\[
\begin{split}
\mathbf{regret}(N)\leq \tilde{R}_1(N)+\tilde{R}_2(N),
\end{split}
\]
where 
\[
\begin{split}
\tilde{R}_1(N)=&\,\left(\dfrac{4(\tilde{D}+\sqrt{2\tilde{C}
((\log_2(N+1)+2)\log2+\log(1/\delta))})}{(1-\gamma)\underline{E}}+\left\|\dfrac{d_\rho^{\pi^\star}}{\rho}\right\|_\infty\right)\\
&\times
(N+1)^{\frac{5}{6}}(\log_2 (2N+3))^2,\\
\tilde{R}_2(N)=&\,\frac{\gamma(\log_2(N+1)+1)^2}{1-\gamma}.
\end{split}
\]
Here the constants are $\underline{E}=\frac{\underline{C}^{\alpha}(1-\gamma)^2}{16S^2A^2}$, and 
\[
\begin{split}
\tilde{D}&=(1-\gamma)^6\left(\frac{1}{(1-\gamma)^2}+
\bar{\lambda}\right)^2+\frac{1}{256}(1-\gamma)^6
\beta_{\bar{\lambda}}\left(\frac{32}{(1-\gamma)^4}+\bar{V}_b\right)
+\frac{1}{1-\gamma}+\log(2A),\\
\tilde{C}&=
\frac{\beta_{\bar{\lambda}}^2(1-\gamma)^{12}\left(\frac{(1+B(1-\gamma))}
{(1-\gamma)^2}+\bar{\lambda}\right)^4}{8192}+
\frac{1}{2}(1-\gamma)^6\left(\frac{(1+B(1-\gamma))}{(1-\gamma)^2}+
\bar{\lambda}\right)^4.
\end{split}
\] Here $\bar{V}_b$ is the variance bound defined in Lemma
\ref{verify_basic_reinforce_base}.

Suppose in addition that we specify $\underline{C}^{\alpha}=1/(2\beta_{\bar{\lambda}})$,
then we can simplify the regret bound into the following simple form:
\[
\mathbf{regret}(N)=O\left( \left(\frac{S^2A^2}{(1-\gamma)^{7}}
+\left\|\dfrac{d_{\rho}^{\pi^\star}}{\rho}\right\|_{\infty}\right)
N^{\frac{5}{6}}(\log(N/\delta))^{\frac{5}{2}}\right).
\]
In addition, we also have 
\[
\lim_{N\rightarrow\infty}\mathbf{regret}(N)/(N+1)=0\quad \text{almost surely}
\]
with an asymptotic rate of $O\left(\left(\frac{S^2A^2}{(1-\gamma)^{7}}
+\left\|d_{\rho}^{\pi^\star}/\rho\right\|_{\infty}\right)
N^{-\frac{1}{6}}(\log N)^{\frac{5}{2}}\right)$.
\end{corollary}

\begin{remark}
Notice that here (and below),
with the specific choices of algorithm hyper-parameters and gradient estimators
we are finally able to
make all the problem dependent quantities (\eg, $\gamma,\,S,\,A$, etc.) explicit in the 
big-$O$ notation, which is consistent with the convention of the reinforcement learning 
literature.  
Here the only hidden quantities are some absolute constants. 
\end{remark}

\section{Mini-batch phased policy gradient method}\label{minibatch_PG_phased}
Here we formalize the mini-batch version of Algorithm \ref{PG_phased}
described at the beginning of \S \ref{mini-batch-reinforce} as Algorithm \ref{PG_phased_batch},
and provide a formal statement as well as a proof for Corollary \ref{regret_reinforce_batch}.

Firstly, we have the following lemma, which transfers guarantees on
$\widehat{\nabla}_{\theta}^{(i)}L_{\lambda^l}(\theta^{l,k})$ ($i=1,\dots,M$) to the
averaged gradient estimation 
$\frac{1}{M}\sum_{i=1}^M\widehat{\nabla}_{\theta}^{(i)}
L_{\lambda^{l}}({\theta^{l,k}})$. The proof follows directly from the fact that
the variance of the sum of independent random variables is the sum of the
variances, and is thus omitted.
\begin{lemma}\label{mini-batch-assumption2}
Suppose that 
each
$\widehat{\nabla}_{\theta}^{(i)}L_{\lambda^l}(\theta^{l,k})$ ($i=1,\dots,M$) is
computed using \eqref{reinforce_grad_base} with the corresponding trajectory, and
that the same assumptions as in Lemma \ref{verify_basic_reinforce_base} hold.
Then
Assumption \ref{unbiased-boundedvar} also holds for
$\widehat{\nabla}_{\theta}L_{\lambda^l}(\theta^{l,k})
=\frac{1}{M}\sum_{i=1}^M\widehat{\nabla}_{\theta}^{(i)}
L_{\lambda^{l}}({\theta^{l,k}})$ with the same constants
$C,C_1,C_2,M_2,\delta_{l,k}$ and $\bar{V}_b$ as in
Lemma \ref{verify_basic_reinforce_base}, while
$M_1=\frac{32}{(1-\gamma)^4}+\frac{\bar{V}_b}{M}$.
\end{lemma}

Now we are ready to state and prove (a more formal version of) 
Corollary \ref{regret_reinforce_batch}.
\begin{corollary}[Formal statement of Corollary \ref{regret_reinforce_batch}]
\label{regret_reinforce_batch_formal}
Under Assumption \ref{softmax-reg},
suppose that for each $l\geq 0$, we choose
$\alpha^{l,k}=C_{l,\alpha}\frac{1}{\sqrt{k+3}\log_2(k+3)}$, with
$C_{l,\alpha}\in[\underline{C}^\alpha, 1/(2\beta_{\lambda^l})]$,
$\underline{C}^\alpha\in(0,1/(2\beta_{\bar{\lambda}})]$ and
$\bar{\lambda}=\frac{1-\gamma}{2}$, and choose
$T_l=2^l$, $\epsilon^l=T_l^{-1/6}=2^{-l/6}$,
$\lambda^l=\frac{\epsilon^l(1-\gamma)}
{2}$ and $\epsilon_{\rm pp}=1/(2A)$. 
In addition, suppose that the assumptions in Lemma \ref{mini-batch-assumption2} hold
(note that Assumption \ref{softmax-reg} and $\lambda^l\leq \bar{\lambda}$
already automatically hold by the other assumptions).
Then we have
 for any $\delta\in(0,1)$, with probability at least $1-\delta$, 
for any $N\geq M\geq 1$,
we have  that
(for the mini-batch Algorithm \ref{PG_phased_batch})
\[
\begin{split}
\mathbf{regret}(N-M;M)\leq \tilde{R}_1(N;M)+\tilde{R}_2(N;M),
\end{split}
\]
where 
\[
\begin{split}
\tilde{R}_1(N;M)=&\left(\dfrac{4(\tilde{D}_M+\sqrt{2\tilde{C}
((\log_2(N/M)+2)\log2+\log(1/\delta))})}{(1-\gamma)\underline{E}}+\left\|\dfrac{d_\rho^{\pi^\star}}{\rho}\right\|_\infty\right)\\
&\,\times M^{\frac{1}{6}}N^{\frac{5}{6}}(\log_2 (2(N/M)+1))^2,\\
\tilde{R}_2(N;M)=&\frac{\gamma M(\log_2(N/M)+1)^2}{1-\gamma}.
\end{split}
\]
Here the constants $\underline{E}$ and $\tilde{C}$  are the same as in
Corollary \ref{regret_reinforce_formal}, while
\[
\tilde{D}_M=(1-\gamma)^6\left(\frac{1}{(1-\gamma)^2}+
\bar{\lambda}\right)^2+\frac{1}{256}(1-\gamma)^6
\beta_{\bar{\lambda}}\left(\frac{32}{(1-\gamma)^4}+\frac{\bar{V}_b}{M}\right)
+\frac{1}{1-\gamma}+\log(2A).
\]
Here $\bar{V}_b$ is the variance bound defined in Lemma
\ref{verify_basic_reinforce_base}.

Suppose in addition that we specify $\underline{C}^{\alpha}=1/(2\beta_{\bar{\lambda}})$. 
Then we can simplify the 
regret bound into the following simple form:
 \[
 \begin{split}
& \mathbf{regret}(N;M)\\
&=O\left(\left(\dfrac{S^2A^2}{(1-\gamma)^7}+\left\|\dfrac{d_{\rho}^{\pi^\star}}{\rho}\right\|_{\infty}\right)
(M^{\frac{1}{6}}+M^{-\frac{5}{6}})(N+M)^{\frac{5}{6}}(\log(N/\delta))^{\frac{5}{2}}
+\dfrac{M(\log N)^2}{1-\gamma}\right).
\end{split}
\]
 In addition, we also have 
 \[
\lim_{N\rightarrow\infty}\mathbf{regret}(N;M)/(N+1)=0 \quad \text{almost surely}
\]
with an asymptotic rate of 
\[
O\left(\left(\dfrac{S^2A^2}{(1-\gamma)^7}+\left\|\dfrac{d_{\rho}^{\pi^\star}}{\rho}\right\|_{\infty}\right)
(M^{\frac{1}{6}}+M^{-\frac{5}{6}})N^{-\frac{1}{6}}\left(1+\dfrac{M}{N}\right)^{\frac{5}{6}}(\log N)^{\frac{5}{2}}
+\dfrac{M(\log N)^2}{(1-\gamma)N}\right).
\]
\end{corollary}

\begin{proof}[Proof of Corollary \ref{regret_reinforce_batch}]
By the definition of $\mathbf{regret}(N;M)$, we immediately see that
\BEQ\label{two_regrets}
\mathbf{regret}(N;M)\leq M\mathbf{regret}(\lfloor N/M\rfloor),
\EEQ
where $\mathbf{regret}(J)$ ($J\geq 0$) is the original regret \eqref{regret_def} in the mini-batch setting, which
is defined only for the total number of inner iterations/steps (instead of episodes, so not magnified
with a factor of $M$). More precisely, we have that for any $J\geq 0$,
\[
\mathbf{regret}(J)=\sum\nolimits_{\{(l,k)|B_{\mathcal{T}}(l,k)\leq J\}}F^\star-F(\pi_{\theta^{l,k}}).
\]

Now by Lemma \ref{mini-batch-assumption2} and Theorem \ref{overall_regret}, and following
the same simplification as is done in Corollary \ref{regret_reinforce_formal}, we have that  
for any
$J\geq 0$,
\[
\begin{split}
&\mathbf{regret}(J)\\
&\leq\left(\dfrac{4(\tilde{D}_M+\sqrt{2\tilde{C}
((\log_2(J+1)+2)\log2+\log(1/\delta))})}{(1-\gamma)\underline{E}}+\left\|\dfrac{d_\rho^{\pi^\star}}{\rho}\right\|_\infty\right)
(J+1)^{\frac{5}{6}}(\log_2 (2J+3))^2\\
&\quad+
\frac{\gamma(\log_2(J+1)+1)^2}{1-\gamma},
\end{split}
\]
where the constants are as stated in the Corollary claims. 

The proof is then complete by plugging the bound of $\mathbf{regret}(J)$ above into
\eqref{two_regrets} and invoking Lemma \ref{highprob2as}. 
\end{proof}

{\linespread{1.29}
\begin{algorithm}[h]
\caption{\textbf{Mini-Batch Phased Policy Gradient Method}}
\label{PG_phased_batch}
\begin{algorithmic}[1]
\STATE {\bfseries Input:} initial parameter $\tilde{\theta}^{0,0}$,
step-sizes $\alpha^{l,k}$, regularization parameters $\lambda^l$, phase lengths
$T_l$ ($l,\,k\geq 0$), post-processing tolerance $\epsilon_{\rm pp}$ and batch size $M>0$.
\STATE Set $\theta^{0,0}=\texttt{PostProcess}(\tilde{\theta}^{0,0},\epsilon_{\rm pp})$.
\FOR{phase $l=0, 1, 2,\dots$}
\FOR{step $k=0,1,\dots,T_l-1$}
\STATE Choose $H^{l,k}$, sample IID trajectories $\{\tau_i^{l,k}\}_{i=1}^M$
(each with horizon $H^{l,k}$)
from $\mathcal{M}$  following policy $\pi_{\theta^{l,k}}$, and
 compute an approximate gradient $\widehat{\nabla}_{\theta}^{(i)}
 L_{\lambda^{l}}({\theta^{l,k}})$
 of $L_{\lambda^l}$ for each trajectory $\tau_i^{l,k}$ ($i=1,\dots,M$).
\STATE Update $\theta^{l,k+1}=\theta^{l,k}+\alpha^{l,k}\frac{1}{M}\sum_{i=1}^M
\widehat{\nabla}_{\theta}^{(i)}
L_{\lambda^{l}}({\theta^{l,k}})$. 
\ENDFOR
\STATE Set $\theta^{l+1,0}=\texttt{PostProcess}(\theta^{l,T_l},\epsilon_{\rm pp})$.
\ENDFOR
\end{algorithmic}
\end{algorithm}
}

\section{Related work}\label{related_work}
Policy gradient methods are a large family of algorithms for
reinforcement learning that directly operate on the agent policy, rather than on the action-value function \cite{glynn1986stochastic,
sutton2018reinforcement}. Examples of policy gradient methods include REINFORCE \cite{REINFORCE},
A3C \cite{A3C}, DPG \cite{DPG}, PPO
\cite{PPO}, and TRPO \cite{TRPO}, to name just a few.
These methods seek to directly maximize the cumulative reward as a function of the policies,
they are straightforward to implement and are amenable to function approximations.
The asymptotic convergence of
(some) policy gradient methods to a stationary point has long been established 
\cite{sutton2000policy, actor-critic, marbach2001simulation, baxter2001infinite}.
The rate of such convergence is also known and has been improved more recently, 
with the help of variance reduction
\cite{SVRPG, xu2020improved, xu2019sample}, Hessian information \cite{HAPG}
and momentum techniques \cite{AdamPG, SRMPG, prox_PG, momentum_PG}. In contrast, 
until recently, 
the global convergence of (vanilla) policy gradient methods (like REINFORCE) 
had not been established unless unrealistic assumptions
like concavity of the expected cumulative reward function \cite{ma2016online}
are imposed. 
 The only 
exceptions are TRPO \cite{neu2017unified} and the soft-max natural policy gradient method with 
fully known models \cite{kakade_2019}, 
which happen
to be equivalent to the MDP Expert algorithms \cite{even2004experts, 
even2009online, neu2010online}.  

In the past two years, a line of research on the global convergence theory for 
(both vanilla and natural) policy gradient
methods has emerged.
By using a gradient domination property of the cumulative reward,
global convergence of (both vanilla and natural) policy gradient methods is first 
established for linear-quadratic
regulators \cite{PG_LQR}. For general Markov
Decision Processes (MDPs), \cite{zhang2019global} establishes convergence to
approximately locally optimal (\ie, second-order stationary) solutions for
vanilla policy gradient methods.
The global optimality of stationary points for general
MDPs is shown in \cite{russo_PG}, and rates of convergence towards
globally optimal solutions for (both vanilla and natural) policy gradient methods 
with (neural network) function approximation
 are derived in \cite{kakade_2019, neural_PG}. These convergence results are then
 improved by several very recent works focusing
 on exact gradient settings. In particular, \cite{softmax_PG}
 focuses on the more practically relevant soft-max parametrization and vanilla policy gradient
 methods and improves the results
 of \cite{kakade_2019} by removing the requirement of the relative entropy regularization and
 obtaining faster convergence rates; \cite{russo_note_2020} obtains linear
  convergence for a general class of policy gradient methods;  \cite{cen2020fast}
  shows local quadratic convergence of natural policy gradient methods; and 
  \cite{zhang2020variational} 
  extends the results to reinforcement learning with general utilities.  
  For more modern
  policy gradient methods, 
  \cite{zhao2019stochastic} establishes the
  asymptotic global
  convergence of of TRPO, 
  while \cite{liu2019neural} further derives the global convergence rates for PPO
  and TRPO. These rates are then improved in \cite{shani2019adaptive} for TRPO with adaptive
  regularization terms.  Very recently, \cite{fu2020single} extends these results to obtain the
  global convergence rates of single-timescale actor-critic methods with PPO actor updates, 
  and \cite{PC-PG} derives global convergence rates of a new policy gradient algorithm combining
  natural policy gradient methods with a policy cover technique and show that the algorithm entails
  better exploration behavior and hence removes the necessity
  for the access to a fully supported initial distribution $\rho$, which is
assumed in most other works on
  global convergence of policy gradient methods (including our work). 
   All the above works either require exact and deterministic updates
  or 
  mini-batch updates 
  with a diverging mini-batch size. 
  
  Lately, \cite{jin2020analysis} studies vanilla policy gradient methods using the REINFORCE 
  gradient estimators computed with a single trajectory in each episode and obtains high probability
  sample complexity results, but the setting is restricted
  to linear-quadratic regulators and their bounds have polynomial dependency on $1/\delta$ 
  (in contrast to our logarithmic dependency on $1/\delta$), where $\delta$ is the probability that the 
  bounds are violated.  
  The authors of \cite{POLITEX} study natural policy gradient methods with a general high
  probability estimation oracle for state-action value functions (\ie, $Q$-functions) in the average
  reward
  settings, and establish
  high probability regret bounds for these algorithms.  
  Finally, we remark that there are also some recent results on the global convergence rates of
  natural policy gradient methods in adversarial settings (with full-information feedback)
  \cite{OPPO}, 
  model-based natural policy gradient methods \cite{efroni2020optimistic} as well as 
  extensions to non-stationary \cite{fei2020dynamic} and  
  multi-agent game settings \cite{zhang2019policy, mazumdar2019policy, carmona2019linear, 
  fu2019actor, guo2020general},  which
  are beyond the scope of this paper.

\end{document}